\documentclass[twoside]{article}

\usepackage[accepted]{aistats2026}
\usepackage[round]{natbib}
\AtBeginDocument{%
  \renewcommand{\refname}{References}

}

\usepackage{amsfonts}
\usepackage{amsmath}
\usepackage{amsthm}
\usepackage{amssymb}
\usepackage{bm}
\usepackage[english]{babel}
\usepackage[utf8]{inputenc}
\usepackage{csquotes}
\usepackage{mathtools}

\usepackage{color}
\usepackage{xcolor}
\usepackage{microtype}
\usepackage{graphicx}
\usepackage{subcaption}
\usepackage{float}
\usepackage{booktabs}
\usepackage[colorlinks,citecolor=blue,urlcolor=blue,linkcolor=blue,linktocpage=true]{hyperref}
\usepackage{comment}
\usepackage{cleveref}
\usepackage{enumitem}  
\usepackage{pifont}
\usepackage{tabularx}

\usepackage{algorithm}
\usepackage{algorithmic}

\usepackage{eqparbox}

\usepackage[T1]{fontenc}

\theoremstyle{definition}

\theoremstyle{plain}
\newtheorem{theorem}{Theorem}
\theoremstyle{plain}
\newtheorem{lemma}{Lemma}
\theoremstyle{plain}
\newtheorem{assump}{Assumption}
\theoremstyle{plain}
\newtheorem{corollary}{Corollary}
\theoremstyle{plain}
\newtheorem{proposition}{Proposition}
\theoremstyle{remark}
\newtheorem*{remark}{Remark}
\theoremstyle{remark}

\usepackage{tikz}
\usetikzlibrary{arrows.meta, positioning}

\begin{document}

%

%

\twocolumn[

\aistatstitle{Provable Accelerated Bayesian Optimization with Knowledge Transfer}

\aistatsauthor{ Haitao Lin \And Boxin Zhao \And  Mladen Kolar \And Chong Liu }

\aistatsaddress{University of Chicago \And  University of Chicago \And USC \& MBZUAI \And University at Albany, SUNY } ]

\begin{abstract}
We study how to accelerate Bayesian optimization (BO) on a target task by transferring historical knowledge from related source tasks. Existing work on BO with knowledge transfer either lacks theoretical guarantees or achieves the same regret as BO in the non-transfer setting, $\tilde{\mathcal{O}}(\sqrt{T \gamma_f})$, where $T$ is the number of evaluations of the target function and $\gamma_f$ denotes its information gain. In this paper, we propose the DeltaBO algorithm, which builds a novel uncertainty-quantification approach on the difference function $\delta$ between the source and target functions, which are allowed to belong to different Reproducing Kernel Hilbert Spaces (RKHSs). Under mild assumptions, we prove that the regret of DeltaBO is of order $\tilde{\mathcal{O}}(\sqrt{T(T/N+\gamma_\delta)})$, where $N$ denotes the number of evaluations from source tasks and typically $N \gg T$. In many applications, source and target tasks are similar, which implies that $\gamma_\delta$ can be much smaller than $\gamma_f$. Empirical studies on both real-world hyperparameter-tuning tasks and synthetic functions show that DeltaBO outperforms other baseline methods and also verify our theoretical claims. Our code is available on GitHub.\footnote{\url{https://github.com/pippahtlin/DeltaBO}}
\end{abstract}

\section{INTRODUCTION}\label{sec:intro}

Modern deep learning models are a driving force behind today’s AI revolution. Training these models requires computer scientists to carefully tune hyperparameters \citep{li2020system}, such as the learning rate and batch size. In materials design \citep{marzari2021electronic}, engineers search for optimal parameters, such as temperature and humidity, to develop new materials that satisfy specific performance criteria. Similarly, in drug discovery \citep{drews2000drug}, researchers must screen vast libraries of small molecules to identify promising drug candidates.

Across these tasks, even with domain expertise, it is extremely challenging to model performance as an explicit function of the input parameters. Owing to its ability to optimize black-box functions, Bayesian optimization (BO) has emerged as a powerful tool in such settings. By enabling adaptive sequential decision-making, BO significantly improves efficiency, precision, and the pace of innovation. 
Recently, \citet{gongora2020bayesian} successfully used BO to find structural parameters that maximized the energy absorption of a structure under compression, reducing the 1{,}800 experiments required by a linear grid search to only 100. In drug discovery, BO offers a principled framework for maximizing a drug candidate's binding efficacy by selecting experimental conditions such as pressure and solution concentration~\citep{korovina2020chembo,shields2021bayesian}. 

Unfortunately, in many applications, even a single evaluation can be costly and time-consuming. For example, completing a 90-epoch training run of the ResNet-50 model on the ImageNet-1k dataset using an NVIDIA M40 GPU requires 14 days \citep{you2018imagenet}, which means that a BO algorithm for this task can run for only 27 iterations in a single year. Even worse, \citet{liang2021scalable} showed that in penicillin production, the widely recognized TuRBO method \citep{eriksson2019scalable} needs 1{,}000 iterations to find the best solution, which amounts to 20 years if each iteration takes one week. 

While advances such as parallelism or multi-fidelity BO can help reduce the evaluation cost or the number of evaluations, these approaches are orthogonal to another important opportunity for acceleration. In practice, we sometimes have historical knowledge about a task, potentially from related tasks, and this can be a valuable information resource. For example, before computer scientists start a new training session for a model, they may already have logged some training checkpoints. Therefore, the key question is: Can we further accelerate the BO process with knowledge transfer from related tasks, ideally in a \emph{provable} way?

In this paper, we answer this question in the affirmative by proposing the DeltaBO algorithm, a BO method that comes with solid theoretical guarantees as well as empirical improvement in a knowledge-transfer setting. We emphasize that our goal is not to outperform state-of-the-art BO systems, but rather to provide a complementary transfer-learning framework with provable guarantees.

\begin{table}[t]
\centering
\caption{Comparison of regret bounds. While Env-GP, Diff-GP, and DeltaBO all address Bayesian optimization with knowledge transfer, our bound clearly highlights the advantage of leveraging related source tasks to accelerate the optimization process.}
\label{tab:regret}
\resizebox{0.49\textwidth}{!}{
\begin{tabular}{ccc}
    \toprule
    \textbf{Algorithms} & \textbf{Regrets} & \textbf{Transfer}\\
    \midrule
    GP-UCB \citep{srinivas2010gaussian} & $\tilde{\mathcal{O}}(\sqrt{T\gamma_f})$  & No \\ 
    Env-GP \citep{shilton2017regret} & $\tilde{\mathcal{O}}(\sqrt{T\gamma_f})$  & Yes \\ 
    Diff-GP \citep{shilton2017regret} & $\tilde{\mathcal{O}}(\sqrt{T\gamma_f})$ & Yes \\ 
    DeltaBO (ours) & $\tilde{\mathcal{O}}(\sqrt{T(T/N+\gamma_\delta)})$  & Yes \\ 
    \bottomrule
\end{tabular}
}
\end{table}

\textbf{Contributions.} Our contributions are as follows:
\begin{itemize}
    \item We systematically study BO with knowledge transfer and propose the DeltaBO algorithm with solid theoretical guarantees.
    \item The regret of DeltaBO is proven to be of order $\tilde{\mathcal{O}}(\sqrt{T(T/N+\gamma_\delta)})$. To the best of our knowledge, this is the first regret bound that shows dependence on $N$, the number of evaluations from source tasks, where typically $N \gg T$. In many applications, source and target tasks are similar, which further implies that $\gamma_\delta \ll \gamma_f$. See Table \ref{tab:regret} for the regret comparison.
    \item Empirical studies on both real-world hyperparameter-tuning tasks and synthetic functions show that DeltaBO outperforms other baseline methods and verify our theoretical claims.
\end{itemize}

\textbf{Technical novelties.} Our technical novelties are as follows:
\begin{itemize}
	\item We adopt an additive decomposition of the target function. \citet{poloczek2017multi} proposed a related additive structure in which each information source is modeled as an additive perturbation of a latent objective; however, their focus is multi-information-source optimization, and they do not derive regret bounds that quantify transfer acceleration with respect to the number of source evaluations. While the additive model (eq. \eqref{eq:add-model}) was also used in previous work~\citep{shilton2017regret}, our assumption is strictly more general since the source and difference functions are allowed to come from two \emph{independent} Gaussian Processes (GPs).
	\item At the heart of our algorithm design is a novel uncertainty-quantification approach built on the difference function so that each evaluation of the target function can be interpreted as a biased observation of it, with the bias given by the source function.
	\item Using the Schur complement, we extend the analysis of GP-UCB \citep{srinivas2010gaussian} by proving that the variance sequence is monotonically non-increasing throughout the iterations, which is a key component in bounding the total variance of queries from the unknown source function.
\end{itemize}

\textbf{Notations.} 
We use standard asymptotic notation throughout the paper. 
The notation $\mathcal{O}(f(n))$ denotes a quantity bounded in absolute value 
by a constant multiple of $f(n)$. The notation $\tilde{\mathcal{O}}(f(n))$ 
suppresses logarithmic factors, i.e., 
$\tilde{\mathcal{O}}(f(n)) = \mathcal{O}(f(n)\,\mathrm{polylog}(n))$. 
Finally, $o(f(n))$ denotes a term such that $o(f(n))/f(n) \to 0$ as $n \to \infty$.

\section{RELATED WORK}\label{sec:rw}

\textbf{Without theoretical guarantees.} Over the past decade, transfer learning has emerged as a powerful strategy for accelerating BO by leveraging prior experience from related tasks.
\citet{swersky2013multi} pioneered this line of work by applying multi-task GPs to share information across tasks, \citet{yogatama2014efficient} proposed constructing a response surface from deviations relative to the per-dataset mean, and \citet{poloczek2016warm} introduced a general warm-start framework. Concurrently, \citet{wistuba2016two} proposed a two-stage surrogate model that approximates response functions and then combines them by similarity. Later, \citet{wistuba2018scalable} introduced a scalable GP framework that weights surrogates via product of experts or kernel regression, and \citet{feurer2018practical} proposed a ranking-weighted GP ensemble. Building on these ideas, \citet{perrone2018scalable} proposed a multi-task adaptive Bayesian linear regression surrogate for efficient transfer in hyperparameter optimization. Extending this line of work, \citet{perrone2019learning} reframed transfer as \emph{search space design}, learning task-adaptive reduced spaces that guide BO to promising regions from historical tasks, especially in large-scale settings.

In parallel, \citet{law2019hyperparameter} enabled knowledge transfer by learning the shared representation of the training data. Following this work, \citet{li2022transfer2} designed the BO search space by combining promising regions with the voting result from the GP classifier. \citet{salinas2020quantile} developed a quantile-based method that leverages Gaussian copulas to model task relationships. More recently, \citet{tighineanu2022transfer} proposed a hierarchical GP framework that provides a principled way to capture shared structure across tasks. For a comprehensive survey of this growing literature, we refer readers to \citet{bai2023transfer}.

\textbf{With theoretical guarantees.} 
While much of the literature on BO with knowledge transfer has focused on heuristic design, only a limited number of works provide theoretical guarantees. For instance, \citet{wang2018regret} established regret bounds for meta Bayesian optimization under unknown GP priors, and \citet{wang2024pre} studied the use of pre-trained GPs for Bayesian optimization, providing both methodological insights and a theoretical characterization of transfer from prior data. Although these works advance the understanding of meta BO, they are less directly related to our setting. The most relevant work is that of \citet{shilton2017regret}, who derived regret bounds for BO with transfer and proposed the Diff-GP method. However, Diff-GP requires the strong assumption that the source and target functions share the same kernel, whereas DeltaBO directly models the difference function and allows it to belong to a distinct reproducing kernel Hilbert space (RKHS). In addition, Diff-GP corrects the bias of every source sample after each target query, which becomes computationally expensive when the number of source samples is large. By contrast, DeltaBO computes the posterior mean and covariance of the source function only once, resulting in a substantially more efficient procedure. In summary, DeltaBO not only relaxes the modeling assumptions but also achieves improved computational scalability while enjoying stronger theoretical guarantees.

Multi-information-source optimization (e.g.~\citet{poloczek2017multi}) considers settings in which an expensive objective can be evaluated directly or approximated by multiple biased, noisy information sources with different query costs. Their approach models each source as the true objective plus a source-specific GP discrepancy and adaptively selects both the design and the source to query. Unlike these frameworks, DeltaBO assumes a fixed historical source dataset that cannot be queried further and focuses on accelerating a single target task rather than adaptively allocating evaluations across fidelities. Thus, our setting is complementary to multi-fidelity BO.

\section{PRELIMINARIES}\label{sec:pre}

In this section, we introduce the background on Bayesian optimization (BO) and GP-UCB, and formalize our problem setting of BO with knowledge transfer.

\subsection{Bayesian Optimization}

In BO, the objective is to identify the global maximizer of an unknown black-box function $f:\mathcal{D} \to \mathcal{Y}$:
\begin{align*}
x_* \in \arg \max_{x \in \mathcal{D}} f(x),
\end{align*}
where $\mathcal{D}$ denotes the input domain or decision set, and $\mathcal{Y} \subseteq \mathbb{R}$ is the range of function values. The function $f$ is considered a black box because its analytical form and derivatives are unavailable.  

Learning about $f$ is possible only through sequential, noisy, zeroth-order evaluations. Over $T$ rounds, the observation at iteration $t \in [T]$ takes the form
\begin{align}\label{eq:feedback}
y_t = f(x_t) + \varepsilon_t,
\end{align}
where $x_t \in \mathcal{D}$ is the query point and $\varepsilon_t \sim \mathcal{N}(0, \sigma^2)$ represents Gaussian observation noise. The performance of a BO algorithm is typically measured by its cumulative regret,
\begin{align*}
R_T = \sum_{t=1}^T \bigl( f(x_*) - f(x_t) \bigr),
\end{align*}
which compares the value of the best query so far to the global optimum. An algorithm is said to be \emph{no-regret} if $\lim_{T \to \infty} R_T/T = 0$. We refer readers to \citet{frazier2018tutorial} for a tutorial on BO.

\subsection{GP-UCB Algorithm}

Gaussian processes (GPs) \citep{williams2006gaussian} provide a flexible prior distribution over functions and form the backbone of many BO methods. 
Formally, a GP is a collection of random variables $\{f(x) : x \in \mathcal{D}\}$ such that any finite subset follows a multivariate Gaussian distribution. 
A GP is fully specified by a mean function $m(x)$ and a positive semidefinite kernel function $k(x,x')$, denoted as
\[
f(x) \sim \mathcal{GP}\!\bigl(m(x), k(x,x')\bigr).
\]

Given noisy observations $\{(x_i, y_i)\}_{i=1}^t$ with $y_i$ generated by eq. \eqref{eq:feedback}, the posterior distribution of $f$ at a new point $x$ is Gaussian with mean and variance
\begin{align*}
\mu_t(x) &= \mathbf{k}_t(x)^\top \bigl(\mathbf{K}_t + \sigma^2 \mathbf{I}_t\bigr)^{-1} \mathbf{y}_{1:t}, \\
\sigma_t^2(x) &= k(x,x) - \mathbf{k}_t(x)^\top \bigl(\mathbf{K}_t + \sigma^2 \mathbf{I}_t\bigr)^{-1} \mathbf{k}_t(x),
\end{align*}
where $\mathbf{K}_t \in \mathbb{R}^{t \times t}$ is the kernel matrix with entries $[k(x_i, x_j)]_{i,j=1}^t$, $\mathbf{k}_t(x) = [k(x_1, x), \dots, k(x_t, x)]^\top$, and $\mathbf{y}_{1:t} = [y_1, \dots, y_t]^\top$ collects the observations up to $t$.

The GP-UCB algorithm \citep{srinivas2010gaussian,chowdhury2017kernelized} leverages this posterior to balance exploration and exploitation. 
At each round $t$, the query point is chosen as
\[
x_t = \arg\max_{x \in \mathcal{D}} \Bigl\{ \mu_{t-1}(x) + \sqrt{\beta_t}\,\sigma_{t-1}(x) \Bigr\},
\]
where $\beta_t > 0$ is a confidence parameter that grows with $t$. 
This selection rule encourages exploration of uncertain regions while exploiting points with high predicted values.  

A central theoretical result is that, under mild kernel assumptions, GP-UCB achieves sublinear cumulative regret
\begin{equation}
\label{eq:regret-GP-UCB}
R_T = \mathcal{O}\!\left(\sqrt{T \, \beta_T \, \gamma_{f,T} }\right),
\end{equation}
where $\gamma_{f,T}$ is the maximum information gain of function $f$ from $T$ observations, defined as
\begin{equation}
\label{eq:inform-gain}
\gamma_{f,T} \coloneqq \max_{A \subseteq \mathcal{D} : |A| = T} \mathbf{I}\!\left( \mathbf{y}_A ; \mathbf{f}_A \right),
\end{equation}
with $\mathbf{f}_A = [f(x)]_{x \in A}$, $\mathbf{y}_A = \mathbf{f}_A + \bm{\varepsilon}_A$, $\bm{\varepsilon}_A \sim \mathcal{N}(0, \sigma^2 \mathbf{I})$, and $\mathbf{I}(\mathbf{y}_A ; \mathbf{f}_A)$ denoting mutual information.  
This guarantee makes GP-UCB a principled and widely adopted algorithm for BO, and it serves as the foundation for our transfer-learning extension.

\subsection{Problem Setup}

We study how BO can be accelerated by leveraging historical knowledge from a related \emph{source} task. Specifically, we assume access to a dataset
\[
\mathcal{S}^{(0)} = \bigl\{ (x_i^{(0)}, y_i^{(0)}) \bigr\}_{i=1}^N,
\]
where the outputs are generated from a \emph{source} function $g:\mathcal{D} \to \mathcal{Y}$ according to
\[
y_i^{(0)} = g(x_i^{(0)}) + \varepsilon_i^{(0)}, \qquad \varepsilon_i^{(0)} \sim \mathcal{N}(0, \sigma_0^2).
\]
Our objective is to design an algorithm that incorporates $\mathcal{S}^{(0)}$ to reduce the regret incurred, thereby accelerating the optimization process when optimizing the \emph{target} function $f$.

To formalize the connection between the source and target tasks, we adopt the following assumption.
\begin{assump}[Additive model]
\label{asm:add}
The target function $f$ can be decomposed into the sum of the source function $g$ and a difference function $\delta$:
\begin{equation}
\label{eq:add-model}
f(x) = g(x) + \delta(x),
\end{equation}
where $g$ and $\delta$ are drawn \emph{independently} from two GPs:
\[
g(x) \sim \mathcal{GP}(0, k_g(x,x')), \quad 
\delta(x) \sim \mathcal{GP}(0, k_\delta(x,x')).
\]
Without loss of generality, we restrict all kernels to be uniformly bounded by $1$.
\end{assump}

Assumption~\ref{asm:add} is mild, as the only essential requirement is the independence of $g$ and $\delta$, which is natural in many applications. Compared with the model assumption used in Diff-GP~\citep{shilton2017regret}, which implicitly requires $g$ and $\delta$ to be governed by GPs with the \emph{same} kernel, our assumption permits $k_g \neq k_\delta$~\footnote{While we study knowledge transfer from only one source task, our model and assumptions can be easily extended to a setting with multiple source tasks by assuming that each source and difference function is drawn from independent GPs.}. This additional flexibility is important for modeling cases in which the source and target functions do not lie in the same Reproducing Kernel Hilbert space (RKHS)~\citep{williams2006gaussian}.  

As we will show in Section~\ref{sec:theoretical-analysis}, when the difference function $\delta$ is easier to learn than the full target function $f$, and when the source dataset $\mathcal{S}^{(0)}$ is sufficiently large, the additive structure in eq.~\eqref{eq:add-model} enables substantial gains in learning efficiency and regret performance.






\section{THE DELTABO ALGORITHM}\label{sec:alg}

We now introduce our proposed DeltaBO algorithm, which efficiently leverages the source dataset to accelerate optimization of the target function. See Figure \ref{fig:diagram} for the workflow diagram and Algorithm \ref{alg:transfer_ucb} for the complete procedure.

\begin{figure}[!htbp]
\includegraphics[width=\linewidth]{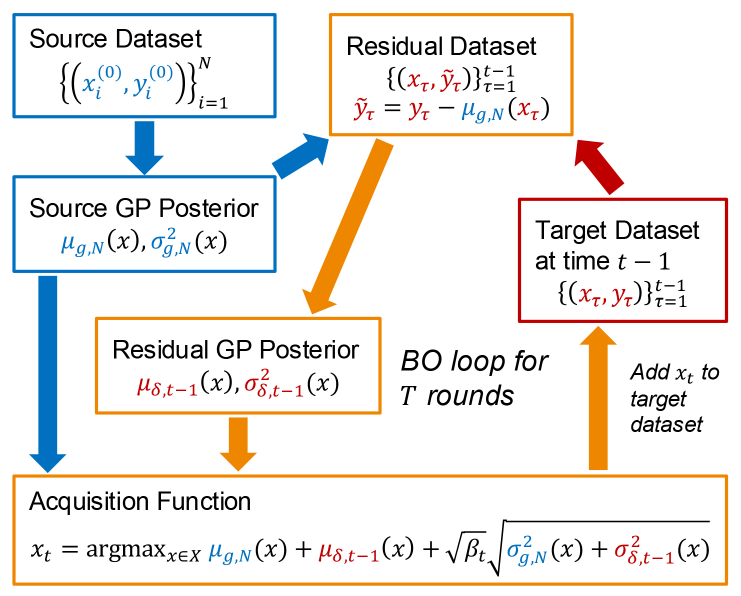}
    \caption{The DeltaBO framework. Information from the source dataset is marked in blue, information from the target task is marked in red, and the algorithmic process is marked in orange, which runs for a total of $T$ rounds.} 
    \label{fig:diagram}
\end{figure}

\begin{algorithm}[!htbp]
\caption{DeltaBO Algorithm}\label{alg:transfer_ucb}
\textbf{Inputs:} Input space $\mathcal{D}$; source posterior mean $\mu_{g,N}(x)$; source posterior variance $\sigma_{g,N}^2(x)$; difference kernel $k_\delta$; noise variance $\sigma^2$; number of iterations $T$.
\begin{algorithmic}[1]
\FOR{$t = 1,2,\dots,T$}
    \STATE Select $x_t$ by eq. \eqref{eq:x-update}.
    \STATE Query the target function $f$ by eq. \eqref{eq:feedback}.\label{algline:queryf}
    \STATE Compute the residual $\tilde{y}_t = y_t - \mu_{g,N}(x_t)$.
    \STATE Update the posterior of $\delta$ using $(x_t, \tilde{y}_t)$ with noise variance $\sigma_{g,N}^2(x_t)+\sigma^2$, obtaining $\mu_{\delta,t}(x)$ and $\sigma_{\delta,t}^2(x)$.\label{algline:deltagp}
\ENDFOR
\end{algorithmic}
\textbf{Output:}  $\hat{x} = \mathcal{U}(x_1,...,x_T)$.
\end{algorithm}

The design of DeltaBO is motivated by a key observation: under the additive structure in eq. \eqref{eq:add-model}, learning the target function $f$ reduces to learning the difference function $\delta$. Each evaluation of $f(x)$ can be interpreted as a biased observation of $\delta(x)$, with the bias given by $g(x)$. Although $g(x)$ is unknown in practice, it can be accurately approximated, together with an associated uncertainty, using the source samples. Incorporating this posterior estimate of $g(x)$ yields more accurate mean and variance estimates for $f(x)$ and ultimately leads to more efficient optimization.  

\textbf{Posterior of the source function.}
We first compute the posterior distribution of the source function $g$ given the dataset $\mathcal{S}^{(0)}$.  
Let $\mathbf{K}_{g,N}$ denote the kernel matrix with entries $\bigl[k_g(x_i^{(0)}, x_j^{(0)})\bigr]_{1 \leq i,j \leq N}$, and define the kernel vector $\mathbf{k}_{g,N}(x) = [k_g(x_i^{(0)}, x)]_{i=1}^N$.  
Then the posterior mean and variance of $g$ at any $x \in \mathcal{D}$ are
\begin{equation}
\label{eq:posterior-mean-var}
\begin{aligned}
\mu_{g,N}(x) &= \mathbf{k}_{g,N}(x)^\top \bigl(\mathbf{K}_{g,N} + \sigma_0^2 \mathbf{I}\bigr)^{-1} \mathbf{y}^{(0)}, \\
\sigma^2_{g,N}(x) &= k_g(x,x) - \mathbf{k}_{g,N}(x)^\top \bigl(\mathbf{K}_{g,N} + \sigma_0^2 \mathbf{I}\bigr)^{-1} \mathbf{k}_{g,N}(x),
\end{aligned}
\end{equation}
where $\mathbf{y}^{(0)} = [y^{(0)}_1, \dots, y^{(0)}_N]^\top$ collects the source observations~\citep{srinivas2010gaussian}.  

\textbf{Residual observations for $\delta$.}
Next, we observe the target outputs $y_t = f(x_t) + \varepsilon_t$ sequentially for $t = 1,\dots,T$, with $\varepsilon_t \sim \mathcal{N}(0, \sigma^2)$, as shown in Algorithm~\ref{alg:transfer_ucb} line~\ref{algline:queryf}. Since $f(x_t) = g(x_t) + \delta(x_t)$ and the posterior of $g$ is fixed after time~0, we define the residual
\begin{equation*}
\widetilde{y}_t := y_t - \mu_{g,N}(x_t).
\end{equation*}
Let $\nu_t := g(x_t) - \mu_{g,N}(x_t)$ and $\eta_t := \nu_t + \varepsilon_t$. Then
\begin{equation*}
\widetilde{y}_t = \delta(x_t) + \eta_t.
\end{equation*}
Since $\nu_t \sim \mathcal{N}(0, \sigma^2_{g,N}(x_t))$, we have
\[
\eta_t \sim \mathcal{N}\!\left(0, \, \sigma^2_{g,N}(x_t) + \sigma^2\right).
\]
Thus, $\widetilde{y}_t$ provides an unbiased but noisy observation of $\delta(x_t)$, with variance inflated by the uncertainty of $g$. We may therefore treat $\{(x_i, \widetilde{y}_i)\}_{i=1}^{t-1}$ as effective observations of $\delta$ when constructing its posterior and making new evaluations at time $t$, as in line~\ref{algline:deltagp} of Algorithm~\ref{alg:transfer_ucb}. 

\textbf{Posterior of the difference function.}
Let $\mathbf{K}_{\delta,t-1}$ be the kernel matrix over $x_1,\dots,x_{t-1}$ with entries $\bigl[k_\delta(x_i, x_j)\bigr]_{1 \leq i,j \leq t-1}$, and define $\mathbf{k}_{\delta,t-1}(x) = [k_\delta(x_i, x)]_{i=1}^{t-1}$.  
Then, by standard GP regression, the posterior mean and variance of $\delta$ at any $x \in \mathcal{D}$ are
\begin{equation}
\label{eq:delta-posterior-mean-var}
\begin{aligned}
& \mu_{\delta,t-1}(x) =\\
& \mathbf{k}_{\delta,t-1}(x)^\top 
   \Bigl(\mathbf{K}_{\delta,t-1} 
   + \bigl(\sigma^2_{g,N}(x) + \sigma^2\bigr)\mathbf{I}_{t-1}\Bigr)^{-1} 
   \widetilde{\mathbf{y}}_{1:t-1}, \\
& \sigma^2_{\delta,t-1}(x) = k_\delta(x, x) -\\
& \mathbf{k}_{\delta,t-1}(x)^\top
   \Bigl(\mathbf{K}_{\delta,t-1} 
   + \bigl(\sigma^2_{g,N}(x) + \sigma^2\bigr)\mathbf{I}_{t-1}\Bigr)^{-1} 
   \mathbf{k}_{\delta,t-1}(x).
\end{aligned}
\end{equation}
where $\widetilde{\mathbf{y}}_{1:t-1} = [\widetilde{y}_1,\dots,\widetilde{y}_{t-1}]^\top$ collects the residual observations.  

\textbf{Acquisition rule.}
The posterior mean and variance in eq. \eqref{eq:delta-posterior-mean-var} provide point estimates and uncertainty quantification for $\delta(x)$. Combining these with the posterior of $g$ in eq. \eqref{eq:posterior-mean-var} and using Assumption~\ref{asm:add}, we obtain point and variance estimates for $f(x)$. Defining $\sigma_{\delta,0}^2(x) = \sigma^2_{g,N}(x) + \sigma^2$ and $\mu_{\delta,0}(x) \equiv 0$ for all $x \in \mathcal{D}$, the GP-UCB acquisition rule becomes
\begin{equation}
\label{eq:x-update}
\begin{aligned}
x_t \in & \arg \max_{x \in \mathcal{D}} \Bigl\{ 
\mu_{g,N}(x) + \mu_{\delta,t-1}(x) 
\\
&\qquad\qquad + \sqrt{\beta_t}\, \sqrt{\sigma^2_{g,N}(x) + \sigma^2_{\delta,t-1}(x)} 
\Bigr\},
\end{aligned}
\end{equation}
where $\mu_{g,N}(x) + \mu_{\delta,t-1}(x)$ serves as the estimate of $f(x)$ and $\sigma^2_{g,N}(x) + \sigma^2_{\delta,t-1}(x)$ as its predictive variance. Here $\beta_t > 0$ is a confidence parameter that grows with $t$.

Finally, we aggregate all observations and produce $\hat{x}$ by drawing uniformly from the set $\{x_{1}, \ldots, x_{T}\}$ after $T$ iterations. Therefore, the output $\hat{x}$ satisfies $f^* - \mathbb{E}[f(\hat{x})] \leq R_T/T$, which is also known as the expected simple regret upper bound. In practice, one can also choose either the last query point $\hat{x}=x_T$, or the query point that returns maximum observation value $\hat{x}=\arg\max_{t=1,...,T} \{y_1,...,y_T\}$, as the output.

\section{THEORETICAL ANALYSIS}\label{sec:theoretical-analysis}

In this section, we establish formal guarantees for the proposed DeltaBO algorithm. Our analysis begins with a regret bound (Theorem~\ref{thm:regret-bound}), which characterizes how the cumulative regret of DeltaBO depends on the number of target evaluations $T$, the number of available source samples $N$, and the information gains associated with both the difference function $\delta$ and the source function $g$. We then investigate conditions under which the information gain $\gamma_{\delta,T}$ is significantly smaller than that of the target function $\gamma_{f,T}$, thereby explaining the advantage of accelerated convergence of DeltaBO using knowledge transfer, relative to standard BO. Together, these results provide theoretical justification for the efficiency and robustness of our approach.

\subsection{Regret Analysis}

We first establish the main regret bound for DeltaBO, showing that leveraging a large source dataset and explicitly modeling the difference function $\delta$ can substantially reduce cumulative regret relative to standard BO. To facilitate the analysis, we introduce the notation
\[
\tau^2 \coloneqq \sup_{x \in \mathcal{D}} k_{\delta}(x,x),
\]
which bounds the variance of the difference kernel on the decision set. 
Recall that $\gamma_{g,N}$ and $\gamma_{\delta,T}$ denote the information gains of $g$ (from $N$ observations) and $\delta$ (from $T$ observations), respectively, as defined in \eqref{eq:inform-gain}. 

For clarity of exposition, we present the analysis for the case in which the decision set $\mathcal{D}$ is finite, and later discuss possible extensions to infinite decision sets.

\begin{theorem}[Cumulative regret bound of DeltaBO]
\label{thm:regret-bound}
Let $\rho \in (0,1)$ denote the error tolerance probability. 
Assume that the decision set $\mathcal{D}$ is finite with cardinality 
$\vert \mathcal{D} \vert$, and that the source dataset 
$\mathcal{S}^{(0)}$ contains $N$ observations of $g$. 
Consider running DeltaBO with
\[
\beta_t = 2\log\!\left(\frac{ \vert \mathcal{D} \vert \, t^2 \pi^2}{6\rho}\right) 
\quad \text{for all } t \geq 1.
\] 
Then, under Assumption~\ref{asm:add}, with probability at least $1-\rho$, 
for all $T \geq 1$, the cumulative regret satisfies
\begin{equation}
\label{eq:regret-bound-precise}
\begin{aligned}
R_T \;\leq\; \Biggl\{ 8T \beta_T & \Biggl( 
\frac{T \gamma_{g,N} \sigma_0^2}{N - 2\gamma_{g,N}} \\
& + C_2 \gamma_{\delta,T} 
   \Bigl(\frac{2 \gamma_{g,N}}{N - 2 \gamma_{g,N}} \sigma_0^2 + \sigma^2\Bigr) 
\Biggr) \Biggr\}^{\!1/2},
\end{aligned}
\end{equation}
where
\[
C_2 = \frac{\tau^{2}/\sigma^{2}}{\log\!\left(1 + \tau^{2}/\sigma^{2}\right)} 
\;\leq\; 1 + \frac{\tau^2}{\sigma^2},
\]
since $x/\log(1+x) \leq 1+x$ for all $x \geq 0$.
\end{theorem}

The proof is deferred to Appendix~\ref{sec:proof-thm-regret-bound}. 
Theorem~\ref{thm:regret-bound} can also be extended to the case in which $\mathcal{D}$ is infinite by applying the standard discretization argument of \citet{srinivas2010gaussian}. We additionally need to assume that $\mathcal{D} \subseteq \mathbb{R}^d$ is compact and convex, and that the GP prior satisfies a high-probability bound on the magnitudes of its partial derivatives \citep{srinivas2010gaussian}. These assumptions are standard requirements in continuous GP bandit analysis. The key idea is to construct a sequence of time-varying finite discretizations that approximate $\mathcal{D}$ with increasing precision, and to show that the additional discretization error is negligible when summed over all rounds. This yields the same order of regret bound as in the finite case, up to a logarithmic adjustment in $\beta_t$ to account for the size of the discretization.

To make the implications of this result more transparent, we state the following corollary, which simplifies the bound in eq.~\eqref{eq:regret-bound-precise} and highlights its asymptotic behavior.

\begin{corollary}
\label{corollary:asymptotic-bd}
Assume $\gamma_{g,N} = o(N)$, $\gamma_{\delta,T} = \mathcal{O}(T)$, and 
$\tau^2 = \mathcal{O}(\sigma^2)$. If
\[
\frac{\gamma_{g,N}}{N} = \mathcal{O}\!\left(\frac{\gamma_{\delta,T}}{T}\right),
\]
then
\begin{equation}
\label{eq:regret-bd-deltabo-corollary}
R_T = \mathcal{O}\!\left( \bigl(\sigma^2 + \sigma_0^2\bigr)^{1/2} 
   \sqrt{\,T \beta_T \gamma_{\delta,T}} \,\right).
\end{equation}
\end{corollary}

\begin{remark}
Recall that the regret bound (eq. \eqref{eq:regret-GP-UCB}) of GP-UCB uses only target data and is of order $\tilde{\mathcal{O}}( \sqrt{ T \gamma_{f,T} } )$. Comparing this with eq. \eqref{eq:regret-bd-deltabo-corollary}, we see that whenever $\gamma_{\delta,T} \ll \gamma_{f,T}$, the regret bound of DeltaBO is substantially smaller than that of the standard GP-UCB algorithm without knowledge transfer. Conversely, negative transfer can occur when $\gamma_{\delta,T}$ is comparable to or larger than $\gamma_{f,T}$. In this regime, the regret bound reverts to the standard GP-UCB rate and, in practice, may even degrade if the prior over-trusts a badly mismatched source.
\end{remark}

The rate at which information gain grows with the number of observations reflects the intrinsic difficulty of learning the function~\citep{williams2006gaussian} and is closely tied to the choice of kernel. Thus, requiring $\gamma_{\delta,T} \ll \gamma_{f,T}$ essentially amounts to assuming that the difference function is easier to learn than the full target function. This condition is natural in practice: it corresponds to the source and target tasks being closely related, which is precisely the regime in which transfer learning is expected to yield the greatest benefit. 

In the next section, we further investigate how information gain grows with the number of iterations for different kernels, and we characterize more precisely when the condition $\gamma_{\delta,T} \ll \gamma_{f,T}$ holds.

\subsection{Information Gain Bounds}
\label{sec:info-gain-bounds}

The regret guarantees in Theorem~\ref{thm:regret-bound} and Corollary~\ref{corollary:asymptotic-bd} depend on the information gain $\gamma_{\delta,T}$ of the difference function $\delta$. Understanding how $\gamma_{\delta,T}$ scales with $T$ is therefore essential. Two factors can make $\gamma_{\delta,T}$ significantly smaller than $\gamma_{f,T}$:  
(i) the spectral decay of the kernel governing $\delta$, which reflects its smoothness or effective complexity; and  
(ii) the amplitude $\tau^2 := \sup_{x \in \mathcal{D}} k_\delta(x,x)$, which rescales the eigenvalues of $k_\delta$ and determines the overall magnitude of the information gain.

We first establish bounds on $\gamma_{\delta,T}$ for three commonly used kernel classes.

\begin{proposition}[Information gain with amplitude scaling]
\label{prop:gamma-rates-tau}
Let $\mathcal{D} \subseteq \mathbb{R}^d$ be compact and $k_\delta(x,x') = \tau^2 \bar{k}_\delta(x,x')$ with $\sup_{x\in \mathcal{D}}\bar{k}_\delta(x,x)\le 1$. Denote by $\gamma_{\delta,T}$ the maximum information gain of $\delta$ after $T$ evaluations under Gaussian noise of variance $\sigma^2$. Then, for absolute constants $C_1,C_2>0$ depending only on $d$ and $\sigma^2$, we have:
\begin{itemize}
\item[\textnormal{(a)}] \textbf{Linear kernel.} If $\bar{k}_\delta(x,x') = x^\top x'$, with effective dimension $d$, then
\[
\gamma_{\delta,T} \;\le\; C_1 \,\tau^2\, d \log\!\big(eT\big) \;+\; C_2 \log\!\big(1+\tau^2\big).
\]
\item[\textnormal{(b)}] \textbf{Squared Exponential (SE) kernel.} 
If $\bar{k}_\delta$ is the squared exponential kernel on $\mathcal{D} \subset \mathbb{R}^d$, i.e.
\[
\bar{k}_\delta(x,x') \;=\; \exp\!\left(-\frac{\|x-x'\|^2}{2\ell^2}\right),
\]
with length-scale parameter $\ell>0$, then
\[
\gamma_{\delta,T} \;\le\; C_1 \,\tau^2\, (\log T)^{d+1} \;+\; C_2 \log\!\big(1+\tau^2\big).
\]
\item[\textnormal{(c)}] \textbf{Matérn kernel.} If $\bar{k}_\delta$ is the Matérn kernel with smoothness parameter $\nu>1$, then
\[
\gamma_{\delta,T} \;\le\; C_1 \,\tau^2 \, T^{\frac{d(d+1)}{2\nu+d(d+1)}} \log T \;+\; C_2 \log\!\big(1+\tau^2\big).
\]
\end{itemize}
\end{proposition}

\begin{remark}[Two drivers of reduced information gain]
\label{rem:two-drivers}
Proposition~\ref{prop:gamma-rates-tau} highlights two complementary mechanisms for keeping $\gamma_{\delta,T}$ small:
\begin{itemize}
\item \textbf{Spectral decay.} If $\delta$ is smoother or of lower effective complexity than the target $f$, the eigenvalues of $\bar{k}_\delta$ decay faster. For example, SE kernels yield $\gamma_{\delta,T}=\mathcal{O}((\log T)^{d+1})$, which grows much more slowly than the polynomial rates of Matérn kernels. Thus, smoothness of $\delta$ ensures that $\gamma_{\delta,T}$ increases slowly with $T$, directly tightening regret bounds.
\item \textbf{Small amplitude $\tau$.} The variance scale $\tau^2$ acts as a multiplicative factor on the leading terms of $\gamma_{\delta,T}$. A small amplitude therefore reduces the overall scale of the information gain, especially in moderate-sample regimes, and further accelerates convergence.
\end{itemize}
In practice, DeltaBO benefits from both effects: spectral decay governs the \emph{growth rate} of $\gamma_{\delta,T}$, while amplitude determines its overall \emph{scale}.
\end{remark}

Substituting the bounds of Proposition~\ref{prop:gamma-rates-tau} into Corollary~\ref{corollary:asymptotic-bd} shows that whenever the difference function $\delta$ is smoother (with faster spectral decay) and/or has small amplitude, the resulting regret bound of DeltaBO is substantially tighter than that of target-only GP-UCB.

\subsection{Theory-Driven Guidelines for Practice}

Our theoretical analysis shows that the cumulative regret of DeltaBO is of order $\tilde{\mathcal{O}}(\sqrt{T(T/N+\gamma_\delta)})$, which suggests several guidelines for applying DeltaBO in practice. First, unlike \citet{shilton2017regret}, our regret bound depends explicitly on $N$, which encourages users to collect more observations from a related source task to improve the performance of DeltaBO. Since our goal is to accelerate the optimization process in $T$, $T$ is typically small, so $N \gg T$ can often be satisfied easily. Second, the dependence on $\gamma_\delta$ reflects the difficulty of the knowledge-transfer problem, as it is determined by the difference function $\delta$. Consequently, DeltaBO performs best when the target task is closely aligned with the source tasks. See our experiments (Section \ref{sec:exp}) for practical results. Finally, while we focus on the accelerated BO regime, collecting more observations from the target task (i.e., increasing $T$) always helps the optimization process since DeltaBO is a no-regret algorithm, i.e., $\lim_{N,T\rightarrow \infty} R_T/T=0$.

\section{EXPERIMENTS}\label{sec:exp}

\subsection{Experimental Settings}

\textbf{Baselines.} We compare DeltaBO with six BO algorithms: GP-EI \citep{jones1998efficient}, GP-PI \citep{kushner1964differential}, GP-TS \citep{thompson1933likelihood}, GP-UCB \citep{srinivas2010gaussian}, Env-GP, and Diff-GP. The first four algorithms are classical BO methods without knowledge transfer, while the last two focus on BO with knowledge transfer and were proposed in \citet{shilton2017regret}. \footnote{We compared against Env-GP and Diff-GP because they are the only GP-UCB-based BO methods with knowledge transfer that come with provable guarantees.}

\textbf{Evaluations.} We summarize performance using two metrics, cumulative regret $R_T$ and average cumulative regret $R_T/T$. For both metrics, lower is better, and we report 95\% confidence intervals, calculated as $\pm 1.96 \cdot \nu / \sqrt{n}$, where $\nu$ denotes the empirical standard deviation of the metrics across $n$ replications. We set $n=100$ for the real-world experiments and $n=30$ in the synthetic setting. Due to the page limit, more details on the experimental settings and results for average regret are deferred to Appendix~\ref{sec:exp-detail}.

\subsection{Real-World AutoML Experiments}
To illustrate the effectiveness of DeltaBO in a real-world setting, we conduct hyperparameter tuning for classification tasks on the UCI Breast Cancer dataset \citep{DuaGraff2017} following \citet{liu2023global}. The black-box objective is defined as the mapping from hyperparameter configurations to validation accuracy, with dimensionality equal to the number of hyperparameters in each classification model. We use two models: Gradient Boosting (GBoost) with 11 hyperparameters, and Multi-Layer Perceptron (MLP) with 8 hyperparameters.

We construct the source and target datasets by first designating 60\% of the data as a shared portion. The remaining 40\% is split evenly between the two domains, so that each dataset contains 80\% of the overall samples. The source dataset contains $N = 90$ observations, and the optimization is run for $T=30$ iterations, starting with 6 initial observations.

For modeling, we use a Matérn kernel for both the source and target functions. Since the source and target tasks are expected to be similar, we model the difference function with a squared exponential kernel, using an appropriate lengthscale to enforce smoothness. This choice provides flexibility in capturing task discrepancies while making the difference function easier to learn.

\begin{figure*}[!htbp]  
\centering
\begin{subfigure}{0.32\textwidth}
    \includegraphics[width=\linewidth]{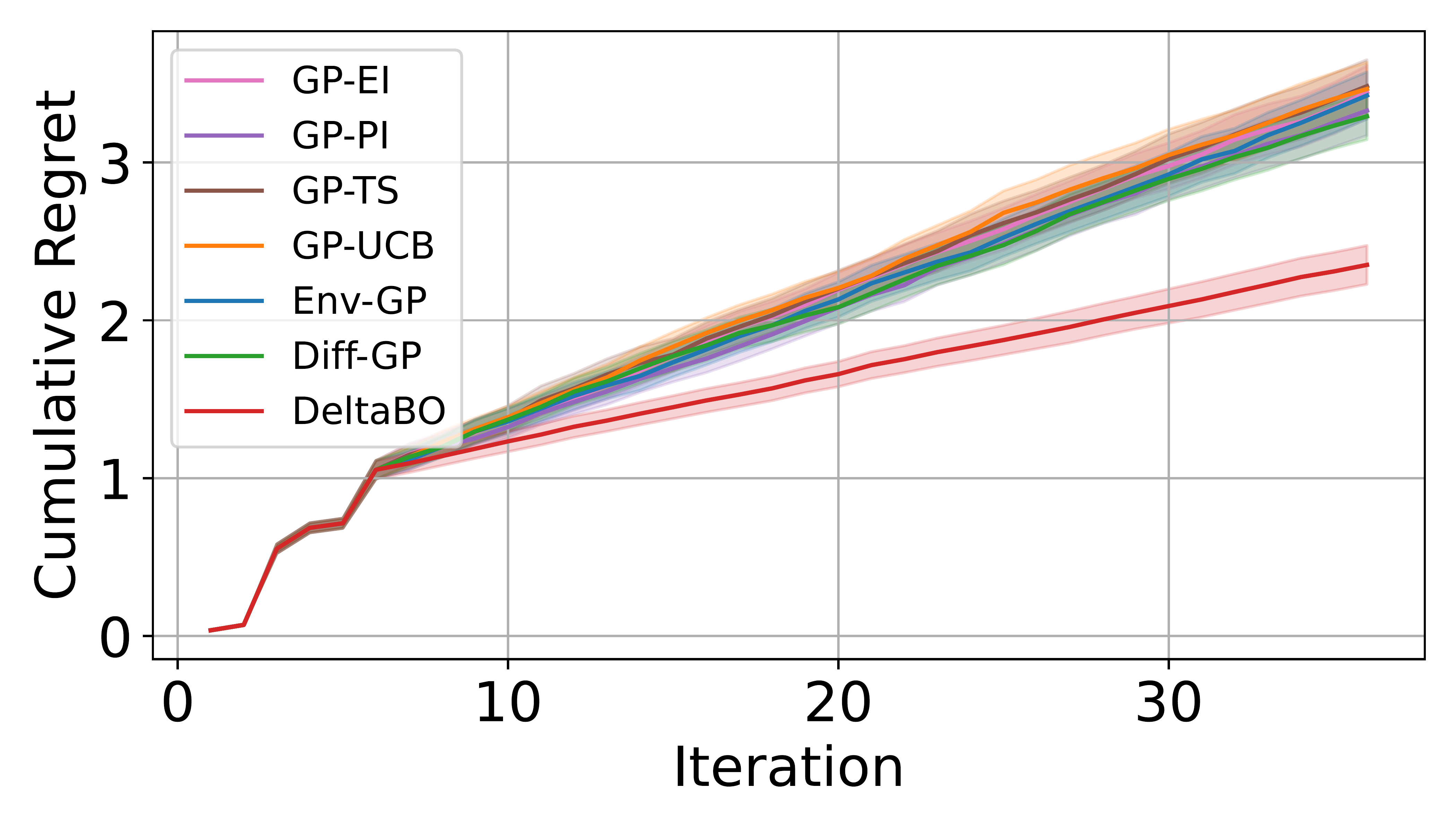}
    \caption{AutoML on GBoost}
    \label{fig:rw-XGB}
\end{subfigure}
\begin{subfigure}{0.32\textwidth}
    \includegraphics[width=\linewidth]{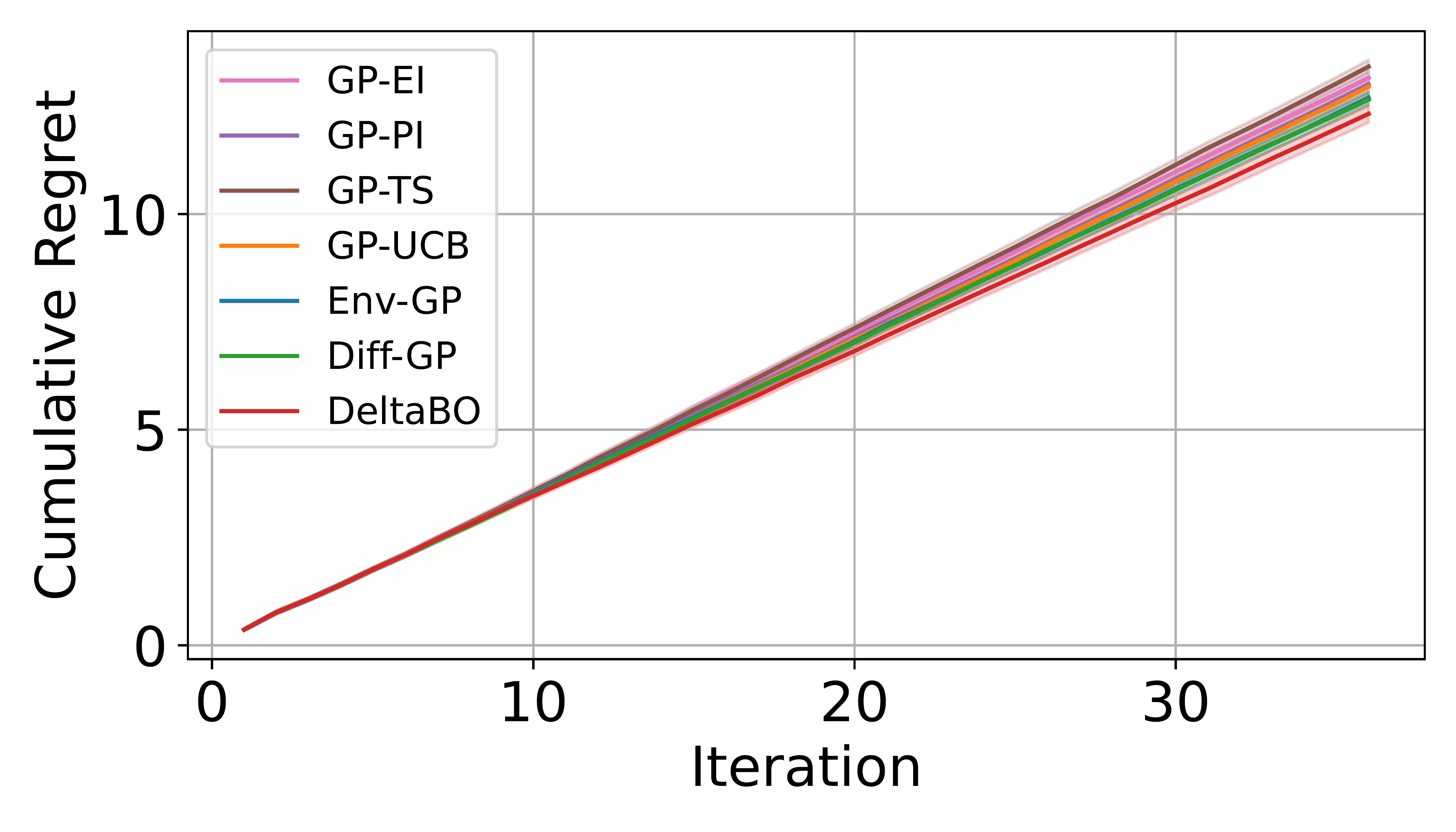}
    \caption{AutoML on MLP}
    \label{fig:rw-mlp}
\end{subfigure}
\begin{subfigure}{0.32\textwidth}
    \includegraphics[width=\linewidth]{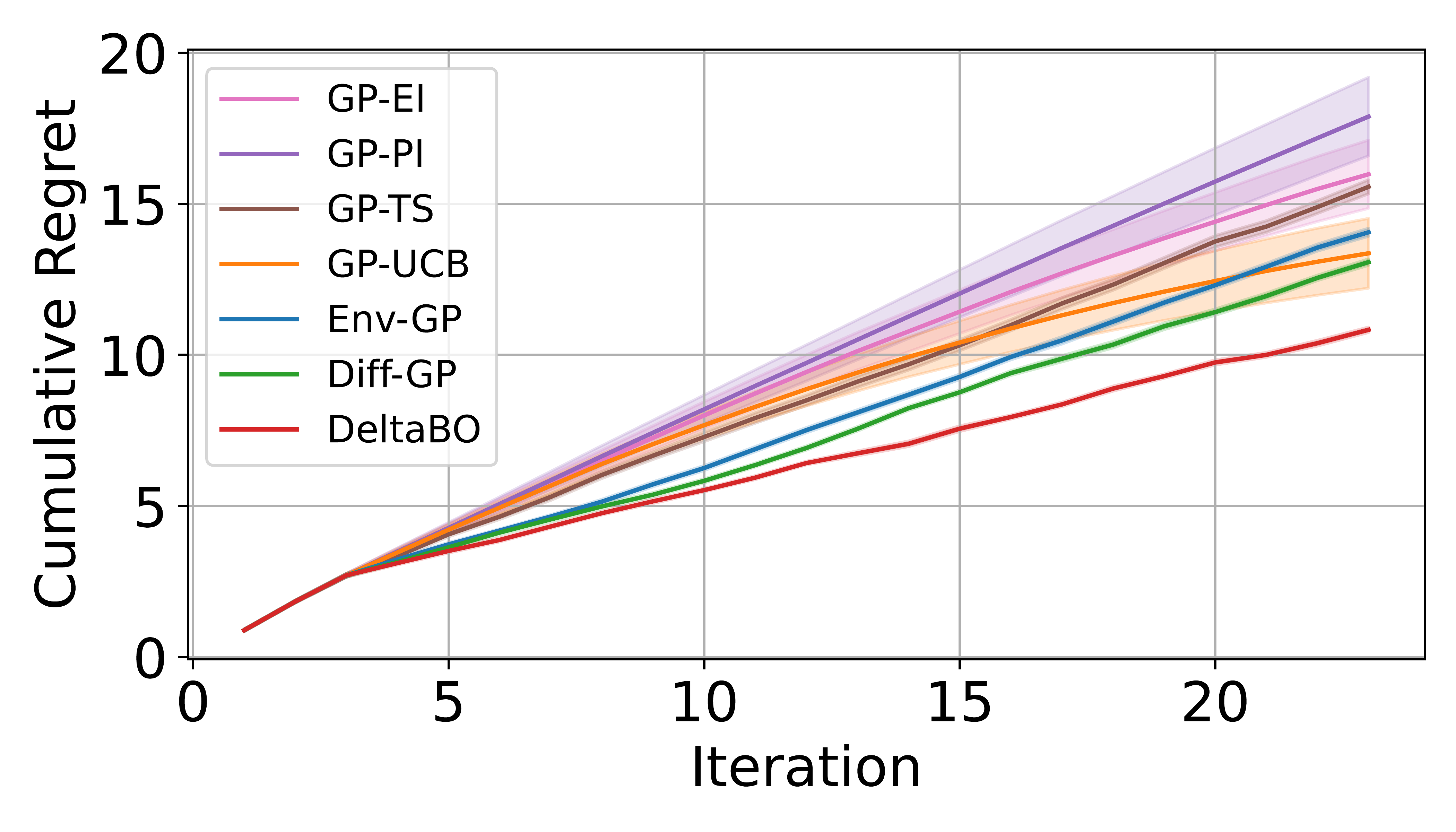}
    \caption{Synthetic Gaussain kernel}
    \label{fig:rw-gau}
\end{subfigure}
\begin{subfigure}{0.32\textwidth}
    \includegraphics[width=\linewidth]{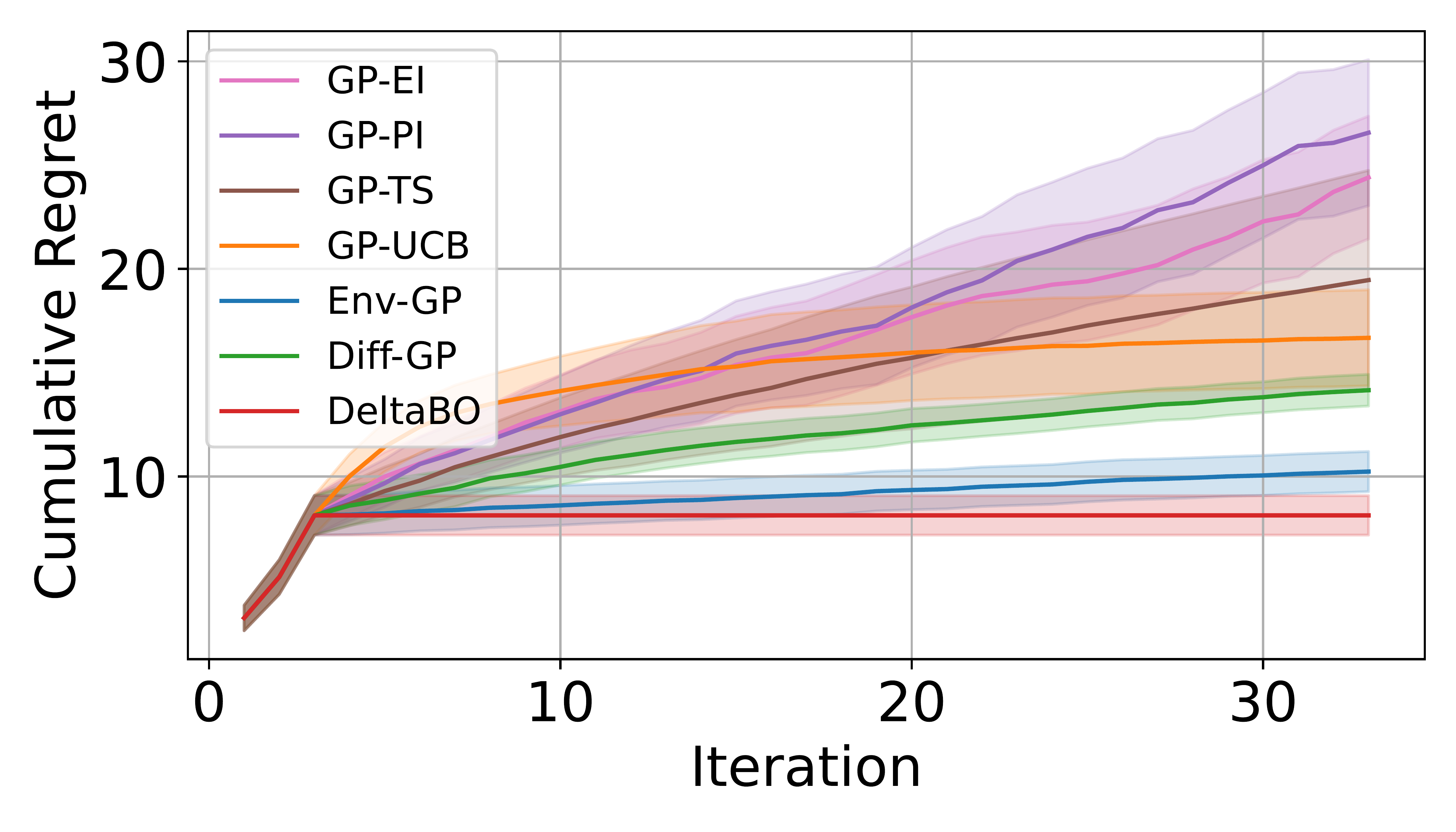}
    \caption{Synthetic Bohachevsky}
    \label{fig:rw-boh}
\end{subfigure}
\begin{subfigure}{0.32\textwidth}
    \includegraphics[width=\linewidth]{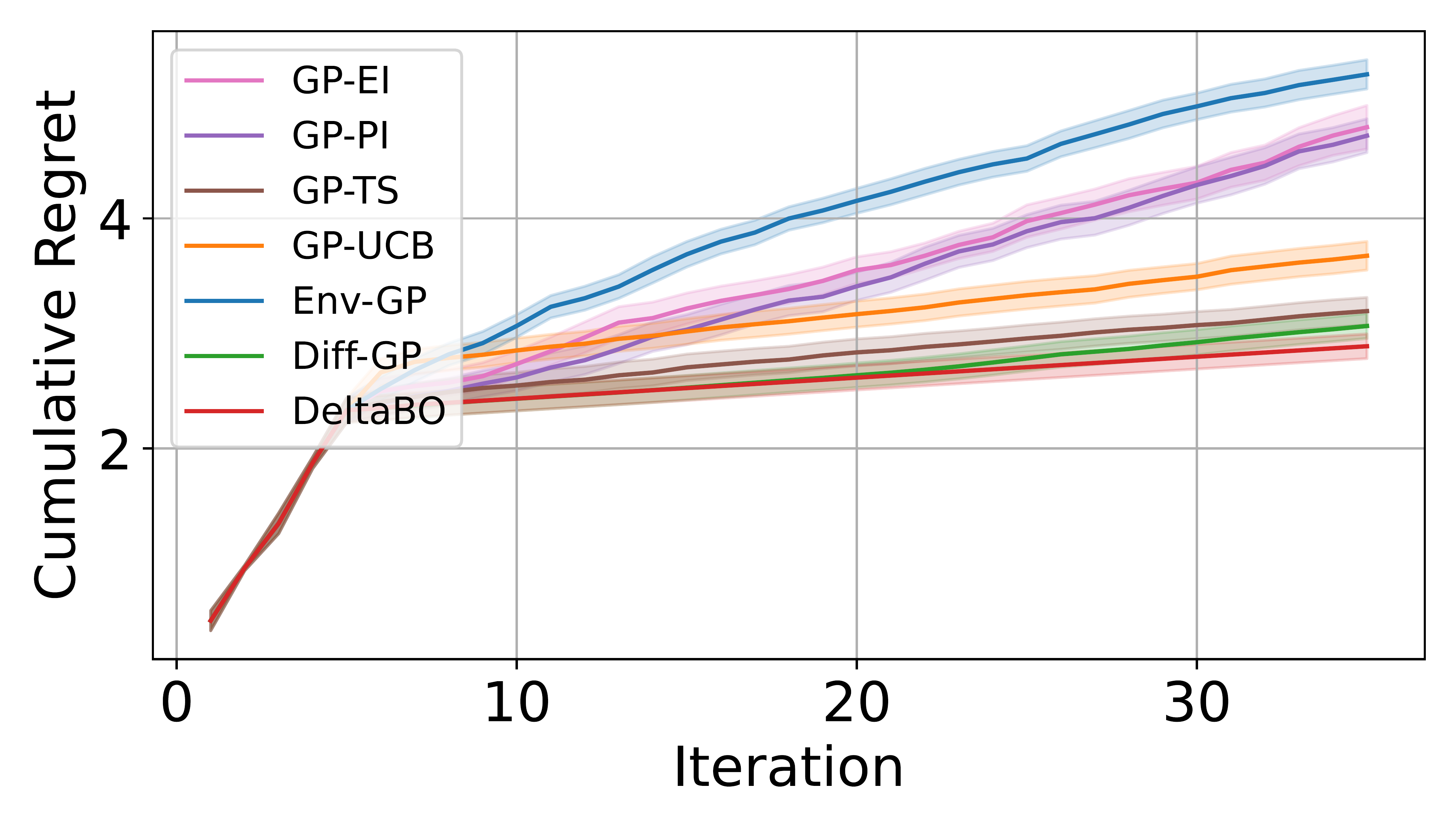}
    \caption{Assumptions-satisfied setting}
    \label{fig:rw-assumption}
\end{subfigure}

\caption{Cumulative regrets of all compared algorithms.}
\end{figure*}

Figure \ref{fig:rw-XGB} shows that DeltaBO significantly outperforms all other baselines by achieving lower cumulative regret. In Figure \ref{fig:rw-mlp}, DeltaBO establishes a consistent advantage after the initial iterations, with only slight overlap in the error bars with Env-GP and Diff-GP, which themselves exhibit nearly identical behavior. This result highlights DeltaBO’s advantage in employing an independent GP with a kernel distinct from those of the source and target objectives, thereby enabling lower regret through appropriate modeling of the difference function.

\subsection{Synthetic Experiments}

\textbf{Gaussian kernel functions.} In the first synthetic experiment, we follow the setup of \citet{shilton2017regret}. The source and target functions are defined as 
\[
g(\mathbf{x}) = \exp\!\left(-\tfrac{1}{2}\lVert \mathbf{x} - \mu' \mathbf{1}\rVert^2\right), 
f(\mathbf{x}) = \exp\!\left(-\tfrac{1}{2}\lVert \mathbf{x} - \mu \mathbf{1}\rVert^2\right) 
\] 
where $\mu' = \mu + \tfrac{s}{\sqrt{n}}$ and $s$ specifies the shift factor between the source and target. We use the squared exponential kernel for all functions and set the dimensionality to 2, as in the original setting. In this case, we focus on a shift of $s = 1$, which corresponds to a more substantial difference between the source and target, and we increase $N$ to 400 in order to better evaluate the algorithms under richer source information.

From Figure~\ref{fig:rw-gau}, we can see that Env-GP and Diff-GP outperform the classical algorithms, but neither method is comparable to DeltaBO. In contrast, DeltaBO achieves significantly lower regret than all other algorithms and maintains small error bars. This demonstrates that DeltaBO benefits the most from the increasing number of source observations, which is consistent with our theoretical conclusion in eq.~\eqref{eq:regret-bound-precise}. Intuitively, Env-GP and Diff-GP are also expected to benefit from additional source data; however, no theoretical guarantee has been established regarding the relationship between increasing source data and decreasing regret for these methods.

\textbf{Bohachevsky functions.} Next, we choose our source and target functions from the Bohachevsky family:
\[
g(x) = x_1^2 + 2 x_2^2 - 0.3 \cos(3 \pi x_1) - 0.4 \cos(4 \pi x_2) + 0.7,
\]
\[
f(x) = x_1^2 + 2 x_2^2 - 0.3 \cos(3 \pi x_1) \cdot \cos(4 \pi x_2) + 0.3,
\]
both defined over $[-2,2]^2$ and discretized into a uniform grid of $120 \times 120$ points.
The source dataset has $N = 400$ samples, while the target evaluation is run for $T = 30$ iterations. We select a Matérn kernel for our target $f$ and difference function $\delta$, and a squared exponential kernel for our source function $g$. 

As shown in Figure~\ref{fig:rw-boh}, DeltaBO outperforms all other algorithms. Env-GP achieves the best performance among the baselines and remains close to DeltaBO in the early iterations, with both methods reaching low regret quickly. As the number of iterations increases, however, DeltaBO sustains a clear advantage: its regret growth slows substantially and eventually levels off, meaning that it has found the global optimum. This behavior highlights DeltaBO’s ability to leverage prior knowledge from $g$, thereby reducing the need for further exploration. Again, DeltaBO achieves lower regret than Diff-GP by allowing independent kernels for each function, which enables DeltaBO to model diverse scenarios, whereas Diff-GP is limited to using the same kernel across functions.

\textbf{Assumption-satisfied setting.} Finally, we study the performance of DeltaBO in a synthetic setting where Assumption \ref{asm:add} is well satisfied. The function domain is set to $[-1,1]^2$, discretized on a uniform $120 \times 120$ grid. We first generate source samples from a Matérn kernel with lengthscale 1.2, and difference samples $\delta$ from a squared exponential kernel with lengthscale 1.0. The target samples are then obtained by eq. \eqref{eq:add-model}. Also, the bound for the difference kernel is set to $\tau^2 = 0.8$ to satisfy Assumption \ref{asm:add}.

From Figure~\ref{fig:rw-assumption}, we observe that Diff-GP achieves performance comparable to DeltaBO in the early iterations. This is consistent with our assumption, since Diff-GP also models the difference between the source and target tasks, making it similar to our additive setting. However, as the number of iterations increases, Diff-GP fails to sustain the same low regret as DeltaBO. We attribute this minor gap to misspecification of the difference kernel, as Diff-GP constrains the kernel of $\delta$ to be the same as those of the source and target, while DeltaBO allows the use of the true kernel with which we generated $\delta$, namely the squared exponential kernel.

Our experiments demonstrate consistent improvements over standard BO baselines in the transfer-learning settings studied here. Extending DeltaBO to substantially larger datasets and integrating it with more heavily engineered BO systems would require additional effort, and we view this as a natural direction for future work.

\section{CONCLUSION}
While BO has been successfully applied to many critical real-world applications, evaluating even a single iteration often remains time-consuming and costly, which severely limits its broader success. To accelerate the optimization process, in this paper we systematically study how BO can be accelerated on a target task by transferring historical knowledge from a related source task. While there are existing works in this area, they either do not come with theoretical guarantees or achieve the same regret bound as BO in the non-transfer setting, failing to show the advantage of access to historical knowledge. To address this problem, we propose the DeltaBO algorithm, in which a novel uncertainty-quantification approach is built on the difference function $\delta$ between the source and target functions, which are allowed to belong to different RKHSs. Under mild assumptions, we prove that the regret of DeltaBO is of order $\tilde{\mathcal{O}}(\sqrt{T(T/N+\gamma_\delta)})$, where typically $N \gg T$ and $\gamma_\delta \ll \gamma_f$ when source and target tasks are similar. Empirical studies on both real-world and synthetic tasks show that DeltaBO outperforms all other baseline methods. Possible future directions include proving a regret lower bound for the DeltaBO algorithm.

\subsubsection*{Acknowledgements}
This work was partially supported by the IBM-UAlbany CEAIS Seed Grant 1201104-1-102522. The authors would like to thank anonymous reviewers and the area chair for helpful comments that improve this paper, and the University of Chicago Data Science Institute Cluster for computing resources.

\bibliography{ref}

\clearpage
\section*{Checklist}





\begin{enumerate}

  \item For all models and algorithms presented, check if you include:
  \begin{enumerate}
    \item A clear description of the mathematical setting, assumptions, algorithm, and/or model. [Yes]
    We provide a clear description of the mathematical setting and assumptions in Sections \ref{sec:pre} and \ref{sec:alg}, where we formalize the additive model (Assumption \ref{asm:add}) and introduce the DeltaBO algorithm, complete with posterior derivations and acquisition rules.
    \item An analysis of the properties and complexity (time, space, sample size) of any algorithm. [Yes]
    In Section \ref{sec:theoretical-analysis}, we analyze the theoretical properties of DeltaBO, including regret bounds that characterize its performance relative to baseline BO algorithms.
    \item (Optional) Source code, with specification of all dependencies, including external libraries. [Yes]
    We provide source code with full dependencies in a Github link.
  \end{enumerate}

  \item For any theoretical claim, check if you include:
  \begin{enumerate}
    \item Statements of the full set of assumptions of all theoretical results. [Yes]
    We state all theoretical assumptions explicitly, most notably Assumption \ref{asm:add}, and the theoretical results about regret analysis is in Section \ref{sec:theoretical-analysis}. 
    \item Complete proofs of all theoretical results. [Yes]
    Complete proofs are provided in the supplementary material (Appendix \ref{sec:proof-thm-regret-bound} and \ref{app:add_proof}).
    \item Clear explanations of any assumptions. [Yes]
    The rationale and implications of our assumptions are explained in Section \ref{sec:pre} and further clarified in Section \ref{sec:theoretical-analysis}, which discusses their role in ensuring sublinear regret.
  \end{enumerate}

  \item For all figures and tables that present empirical results, check if you include:
  \begin{enumerate}
    \item The code, data, and instructions needed to reproduce the main experimental results (either in the supplemental material or as a URL). [Yes]
    We include code, data sources, and instructions to enable full reproducibility of the experiments.
    \item All the training details (e.g., data splits, hyperparameters, how they were chosen). [Yes]
    All training details, including data splits, hyperparameter choices, and kernel settings, are reported in Section \ref{sec:exp} and Appendix \ref{sec:exp-detail}.
    \item A clear definition of the specific measure or statistics and error bars (e.g., with respect to the random seed after running experiments multiple times). [Yes] 
    We clearly define the evaluation metrics and specify how error bars are constructed at the start of Section \ref{sec:exp}.
    \item A description of the computing infrastructure used. (e.g., type of GPUs, internal cluster, or cloud provider). [Yes]
    Details of the computing infrastructure are shown in the appendix.
  \end{enumerate}

  \item If you are using existing assets (e.g., code, data, models) or curating/releasing new assets, check if you include:
  \begin{enumerate}
    \item Citations of the creator If your work uses existing assets. [Yes] We cite all datasets used, including the UCI Breast Cancer dataset, and provide references for benchmark functions used in synthetic experiments.
    \item The license information of the assets, if applicable. [Not Applicable]
    \item New assets either in the supplemental material or as a URL, if applicable. [Not Applicable]
    \item Information about consent from data providers/curators. [Not Applicable]
    \item Discussion of sensible content if applicable, e.g., personally identifiable information or offensive content. [Not Applicable]
  \end{enumerate}

  \item If you used crowdsourcing or conducted research with human subjects, check if you include:
  \begin{enumerate}
    \item The full text of instructions given to participants and screenshots. [Not Applicable]
    \item Descriptions of potential participant risks, with links to Institutional Review Board (IRB) approvals if applicable. [Not Applicable]
    \item The estimated hourly wage paid to participants and the total amount spent on participant compensation. [Not Applicable]
  \end{enumerate}

\end{enumerate}

\newpage
\appendix
\thispagestyle{empty}

\onecolumn
\aistatstitle{Provable Accelerated Bayesian Optimization with Knowledge Transfer: Supplementary Materials}

\section{Proof for Theorem~\ref{thm:regret-bound}}\label{sec:proof-thm-regret-bound}

In this section, we prove Theorem~\ref{thm:regret-bound}. We begin with several 
auxiliary lemmas and then proceed to the main proof of the theorem.

\subsection{Auxiliary Lemmas}

In this section, we present some auxiliary lemmas that will be useful for proving the main theorem.

\begin{lemma}
\label{lemma:confidence-bound}
Fix $\rho \in (0,1)$. By setting
\[
\beta_t \;=\; 2 \log\!\left( \frac{|\mathcal{D}| \pi^2 t^2}{6\rho} \right),
\]
we have, with probability at least $1-\rho$,
\[
\big| f(x) - \mu_{g,N}(x) - \mu_{\delta,t-1}(x) \big|
\;\leq\; \sqrt{\beta_t}\,\sqrt{\sigma^2_{g,N}(x) + \sigma^2_{\delta,t-1}(x)}
\]
for all $t \in \mathbb{N}^+$ and $x \in \mathcal{D}$.
\end{lemma}

\begin{proof}
For any $t \geq 1$ and $x \in \mathcal{D}$, recall that $f(x) = g(x) + \delta(x)$. Define
\[
\tilde{y}_t \;=\; f(x_t) - \mu_{g,N}(x_t) \;=\; \delta(x_t) + \nu_t,
\]
where $\nu_t = g(x_t) - \mu_{g,N}(x_t) \sim \mathcal{N}(0, \sigma^2_{g,N}(x_t))$.  
Conditioned on past observations, we have
\[
\delta(x_t) \sim \mathcal{N}(\mu_{\delta,t-1}(x_t), \sigma^2_{\delta,t-1}(x_t)).
\]
Since $g \perp\!\!\!\perp \delta$, it follows that
\[
f(x_t) - \mu_{g,N}(x_t) - \mu_{\delta,t-1}(x_t)
\;\sim\; \mathcal{N}\!\big(0, \sigma^2_{g,N}(x_t) + \sigma^2_{\delta,t-1}(x_t)\big).
\]

Let
\[
\sigma_t^2(x) \;\coloneqq\; \sigma^2_{g,N}(x) + \sigma^2_{\delta,t-1}(x).
\]
By the Gaussian tail bound,
\[
\Pr\!\left( \,\big| f(x) - \mu_{g,N}(x) - \mu_{\delta,t-1}(x) \big| > \sqrt{\beta_t}\,\sigma_t(x) \right)
\;\leq\; \exp\!\left(-\tfrac{\beta_t}{2}\right).
\]

Applying the union bound over all $x \in \mathcal{D}$ gives
\[
\Pr\!\left( \exists\, x \in \mathcal{D}:\,
\big| f(x) - \mu_{g,N}(x) - \mu_{\delta,t-1}(x) \big| > \sqrt{\beta_t}\,\sigma_t(x) \right)
\;\leq\; |\mathcal{D}| \cdot \exp\!\left(-\tfrac{\beta_t}{2}\right).
\]

To make the guarantee uniform over all $t \in \mathbb{N}^+$, we distribute the total failure probability $\rho$ across time steps.  
Assign the failure probability at time $t$ to be $\tfrac{6\rho}{\pi^2 t^2}$, noting that
\[
\sum_{t=1}^\infty \frac{6}{\pi^2 t^2} = 1.
\]
Thus, it suffices to choose $\beta_t$ such that
\[
|\mathcal{D}| \cdot \exp\!\left(-\tfrac{\beta_t}{2}\right) \;=\; \frac{6\rho}{\pi^2 t^2},
\]
which yields
\[
\beta_t \;=\; 2 \log\!\left( \frac{|\mathcal{D}| \pi^2 t^2}{6\rho} \right).
\]

Finally, applying the union bound over all $t \in \mathbb{N}^+$ establishes that, with probability at least $1-\rho$, the stated inequality holds for all $t \in \mathbb{N}^+$ and all $x \in \mathcal{D}$.
\end{proof}

\begin{lemma}
\label{lem:regret-bound}
Fix $t \geq 1$. Suppose that for all $x \in \mathcal{D}$,
\begin{equation}
\label{eq:confidence-bd-assump}
\big| f(x) - \mu_{g,N}(x) - \mu_{\delta,t-1}(x) \big|
\;\leq\; \sqrt{\beta_t}\,\sqrt{\sigma^2_{g,N}(x) + \sigma^2_{\delta,t-1}(x)}.
\end{equation}
Then the instantaneous regret $r_t := f(x^\star) - f(x_t)$ satisfies
\[
r_t \;\leq\; 2 \sqrt{\beta_t}\,\sqrt{\sigma^2_{g,N}(x_t) + \sigma^2_{\delta,t-1}(x_t)}.
\]
\end{lemma}

\begin{proof}
This proof follows the argument of Lemma~5.2 in~\cite{srinivas2010gaussian}.  
By the definition of GP-UCB, the chosen point $x_t$ maximizes the upper confidence bound:
\[
\mu_{g,N}(x_t) + \mu_{\delta,t-1}(x_t)
+ \sqrt{\beta_t}\,\sqrt{\sigma^2_{g,N}(x_t) + \sigma^2_{\delta,t-1}(x_t)}
\;\geq\;
\mu_{g,N}(x^\star) + \mu_{\delta,t-1}(x^\star)
+ \sqrt{\beta_t}\,\sqrt{\sigma^2_{g,N}(x^\star) + \sigma^2_{\delta,t-1}(x^\star)}.
\]
Using the confidence bound \eqref{eq:confidence-bd-assump}, we have
\[
\mu_{g,N}(x^\star) + \mu_{\delta,t-1}(x^\star)
+ \sqrt{\beta_t}\,\sqrt{\sigma^2_{g,N}(x^\star) + \sigma^2_{\delta,t-1}(x^\star)}
\;\geq\; f(x^\star).
\]
Hence,
\[
f(x^\star) - \mu_{g,N}(x_t) - \mu_{\delta,t-1}(x_t)
\;\leq\; \sqrt{\beta_t}\,\sqrt{\sigma^2_{g,N}(x_t) + \sigma^2_{\delta,t-1}(x_t)}.
\]

On the other hand, applying \eqref{eq:confidence-bd-assump} at $x_t$ gives
\[
\mu_{g,N}(x_t) + \mu_{\delta,t-1}(x_t) - f(x_t)
\;\leq\; \sqrt{\beta_t}\,\sqrt{\sigma^2_{g,N}(x_t) + \sigma^2_{\delta,t-1}(x_t)}.
\]

Combining the two inequalities, we obtain
\begin{align*}
r_t
&= f(x^\star) - f(x_t) \\
&= \big(f(x^\star) - \mu_{g,N}(x_t) - \mu_{\delta,t-1}(x_t)\big)
   + \big(\mu_{g,N}(x_t) + \mu_{\delta,t-1}(x_t) - f(x_t)\big) \\
&\leq 2 \sqrt{\beta_t}\,\sqrt{\sigma^2_{g,N}(x_t) + \sigma^2_{\delta,t-1}(x_t)},
\end{align*}
which proves the claim.
\end{proof}

\begin{lemma}
\label{lemma:variance-decreasing}
Let $\kappa(x,x')$ be a kernel function defined on the domain $\mathcal{D}$. For any positive integer $n$ and any sequence $\{x_t\}_{t \geq 1} \subseteq \mathcal{D}$, define
\begin{align*}
\bm{\kappa}_n(x) &\coloneqq \big[ \kappa(x_1, x), \ldots, \kappa(x_n, x) \big]^{\top} \in \mathbb{R}^n, \quad x \in \mathcal{D}, \\
\mathbf{K}_n &\coloneqq \big[ \kappa(x_i, x_j) \big]_{1 \leq i,j \leq n} \in \mathbb{R}^{n \times n}.
\end{align*}
Further, set
\[
\sigma_n^2(x) \coloneqq \kappa(x,x) - \bm{\kappa}_n(x)^{\!\top} \big( \mathbf{K}_n + \sigma^2 \mathbf{I} \big)^{-1} \bm{\kappa}_n(x), 
\quad x \in \mathcal{D}.
\]
Then, for all $x \in \mathcal{D}$,
\[
\sigma_{n+1}^2(x) \;\leq\; \sigma_n^2(x).
\]
\end{lemma}

\begin{proof}
Fix $n$ and $x\in\mathcal{D}$. Let
\[
\mathbf{A} \coloneqq \mathbf{K}_n+\sigma^2\mathbf{I}\in\mathbb{R}^{n\times n},\qquad
\mathbf{u}\coloneqq \bm{\kappa}_n(x_{n+1})\in\mathbb{R}^n,\qquad
c\coloneqq \kappa(x_{n+1},x_{n+1})+\sigma^2.
\]
Then
\[
\mathbf{K}_{n+1}+\sigma^2\mathbf{I}=\begin{bmatrix}
\mathbf{A} & \mathbf{u} \\[2pt]
\mathbf{u}^\top & c
\end{bmatrix},
\qquad
\bm{\kappa}_{n+1}(x)=\begin{bmatrix}\bm{\kappa}_n(x)\\ \kappa(x_{n+1},x)\end{bmatrix}.
\]
Since $\kappa$ is positive semidefinite and $\sigma^2>0$, $\mathbf{A}$ is positive definite. The Schur complement of the block $\mathbf{A}$ is
\[
s \coloneqq c - \mathbf{u}^\top \mathbf{A}^{-1} \mathbf{u} \;=\; \sigma^2 + \kappa(x_{n+1},x_{n+1}) - \bm{\kappa}_n(x_{n+1})^\top \mathbf{A}^{-1}\bm{\kappa}_n(x_{n+1})
= \sigma^2 + \sigma_n^2(x_{n+1}) \;>\; 0.
\]
By the block inversion formula,
\[
\big(\mathbf{K}_{n+1}+\sigma^2\mathbf{I}\big)^{-1}
=
\begin{bmatrix}
\mathbf{A}^{-1}+\mathbf{A}^{-1} \mathbf{u} s^{-1} \mathbf{u}^\top \mathbf{A}^{-1} & -\mathbf{A}^{-1} \mathbf{u} s^{-1}\\[2pt]
-s^{-1} \mathbf{u}^\top \mathbf{A}^{-1} & s^{-1}
\end{bmatrix}.
\]
Let $\mathbf{v} \coloneqq \bm{\kappa}_n(x)$ and $\alpha\coloneqq \kappa(x_{n+1},x)$. Then
\begin{align*}
\bm{\kappa}_{n+1}(x)^\top\big(\mathbf{K}_{n+1}+\sigma^2\mathbf{I}\big)^{-1}\bm{\kappa}_{n+1}(x)
&= \begin{bmatrix} \mathbf{v}^\top & \alpha \end{bmatrix}
\begin{bmatrix}
\mathbf{A}^{-1}+\mathbf{A}^{-1} \mathbf{u} s^{-1} \mathbf{u}^\top \mathbf{A}^{-1} & -\mathbf{A}^{-1} \mathbf{u} s^{-1}\\[2pt]
-s^{-1} \mathbf{u}^\top \mathbf{A}^{-1} & s^{-1}
\end{bmatrix}
\begin{bmatrix} \mathbf{v} \\ \alpha \end{bmatrix}\\[4pt]
&= \mathbf{v}^\top \mathbf{A}^{-1} \mathbf{v} \;+\; s^{-1}\big(\alpha - \mathbf{u}^\top \mathbf{A}^{-1} \mathbf{v} \big)^2.
\end{align*}
Therefore,
\[
\sigma_{n+1}^2(x)
= \kappa(x,x) - \Big[ \mathbf{v}^\top \mathbf{A}^{-1} \mathbf{v} + s^{-1}\big(\alpha - \mathbf{u}^\top \mathbf{A}^{-1} \mathbf{v} \big)^2 \Big]
= \sigma_n^2(x) - \frac{\big(\alpha - \mathbf{u}^\top \mathbf{A}^{-1} \mathbf{v}\big)^2}{s}.
\]
Since $s>0$, the last term is nonnegative, yielding $\sigma_{n+1}^2(x)\le \sigma_n^2(x)$.
\end{proof}

\begin{lemma}
\label{lem:variance-sum-bound}
Let the maximum mutual information gain be defined as
\[
\gamma_{g,N} \;\coloneqq\; \max_{\substack{A \subseteq \mathcal{D} \\ |A| = N}} I\!\left(y^{(0)}; g_A\right).
\]
Then,
\[
\sum_{t=1}^T \sigma_{g,N}^2(x_t) \;\leq\; \frac{2T \gamma_{g,N}\, \sigma_0^2}{\,N - 2\gamma_{g,N}\,}.
\]
\end{lemma}

\begin{proof}
This proof follows the argument of Lemma~5.3 in~\cite{srinivas2010gaussian}.  

Consider observations of the form
\[
y_i^{(0)} = g(x_i^{(0)}) + \varepsilon_i^{(0)}, 
\quad \varepsilon_i^{(0)} \sim \mathcal{N}(0,\sigma_0^2),
\quad i = 1,\dots,N.
\]
Let $A^{(0)} = \{x_1^{(0)},\dots,x_N^{(0)}\}$. The mutual information between $y_N^{(0)}$ and $g_{A^{(0)}}$ is
\[
I(y_N^{(0)}; g_{A^{(0)}}) \;=\; \tfrac{1}{2} \sum_{i=1}^N \log\!\left(1 + \sigma_0^{-2}\, \sigma^2_{g,i-1}(x_i^{(0)})\right).
\]

By the definition of $\gamma_{g,N}$,
\begin{equation}
\label{eq:gamma-bound}
\sum_{i=1}^N \log\!\left(1 + \sigma_0^{-2}\, \sigma^2_{g,i-1}(x_i^{(0)})\right) 
\;\leq\; 2\gamma_{g,N}.
\end{equation}

Since $\log(1+x) \geq \tfrac{x}{1+x}$ for all $x \geq 0$, we obtain
\[
\log\!\left(1 + \sigma_0^{-2}\, \sigma^2_{g,i-1}(x_i^{(0)})\right) 
\;\geq\; \frac{\sigma^2_{g,i-1}(x_i^{(0)})}{\sigma_0^2 + \sigma^2_{g,i-1}(x_i^{(0)})}.
\]
Plugging this into \eqref{eq:gamma-bound} yields
\[
\sum_{i=1}^N \frac{\sigma^2_{g,i-1}(x_i^{(0)})}{\sigma_0^2 + \sigma^2_{g,i-1}(x_i^{(0)})}
\;\leq\; 2\gamma_{g,N}.
\]

By Lemma~\ref{lemma:variance-decreasing}, $\sigma^2_{g,N}(x) \leq \sigma^2_{g,i-1}(x)$ for all $x \in \mathcal{D}$ and $2 \leq i \leq N$. Thus,
\[
\frac{\sigma^2_{g,N}(x)}{\sigma_0^2 + \sigma^2_{g,N}(x)}
\;\leq\; \frac{\sigma^2_{g,i-1}(x)}{\sigma_0^2 + \sigma^2_{g,i-1}(x)}
\quad \text{for all } i.
\]
Summing over $i=1,\dots,N$ gives
\[
N \cdot \frac{\sigma^2_{g,N}(x)}{\sigma_0^2 + \sigma^2_{g,N}(x)} \;\leq\; 2\gamma_{g,N}.
\]

Rearranging yields
\[
\sigma^2_{g,N}(x) \;\leq\; \frac{2\gamma_{g,N}\,\sigma_0^2}{N - 2\gamma_{g,N}}.
\]

Finally, summing this bound over $t=1,\dots,T$ gives
\[
\sum_{t=1}^T \sigma^2_{g,N}(x_t) \;\leq\; \frac{2T \gamma_{g,N}\,\sigma_0^2}{N - 2\gamma_{g,N}},
\]
as claimed.
\end{proof}

\subsection{Main Proof}

The proof follows the argument of Lemma~5.4 in~\cite{srinivas2010gaussian}.  

From Lemma~\ref{lemma:confidence-bound} and Lemma~\ref{lem:regret-bound}, we know that with probability at least $1-\rho$,  
\[
r_t^2 \;\leq\; 4 \beta_t \bigl(\sigma_{g,N}^2(x_t) + \sigma_{\delta,t-1}^2(x_t)\bigr), 
\qquad t \geq 1.
\]
Summing over $t=1,\dots,T$ gives
\[
\sum_{t=1}^T r_t^2 
\;\leq\; 4 \beta_t \sum_{t=1}^T \bigl(\sigma_{g,N}^2(x_t) + \sigma_{\delta,t-1}^2(x_t)\bigr).
\]

\textbf{The $g$-term.}  
By Lemma~\ref{lem:variance-sum-bound},
\[
\sum_{t=1}^T \sigma_{g,N}^2(x_t) 
\;\leq\; \frac{2T \gamma_{g,N} \sigma_0^2}{N - 2\gamma_{g,N}}.
\]

\textbf{The $\delta$-term.}  
Recall
\[
C_2 \;=\; \frac{\tau^{2}/\sigma^{2}}{\log(1+\tau^{2}/\sigma^{2})}.
\]
For any $s^2 \in [0,\,\tau^2\sigma^{-2}]$ we have
\begin{equation}
\label{eq:C2-ineq}
s^2 \;\leq\; C_2 \log(1+s^2),
\end{equation}
because the function $h(u) = u/\log(1+u)$ is nondecreasing and achieves its maximum at $u=\tau^2\sigma^{-2}$, which equals $C_2$.  

Now set
\[
s^2 \;=\; \sigma_{\delta,0}^{-2}(x_t)\, \sigma_{\delta,t-1}^2(x_t).
\]
Since $\sigma_{\delta,0}^{-2}(x_t) \leq \sigma^{-2}$ and $\sigma_{\delta,t-1}^2(x_t) \leq \tau^2$, we indeed have $s^2 \in [0,\,\tau^2\sigma^{-2}]$, so~\eqref{eq:C2-ineq} applies:
\[
\sigma_{\delta,0}^{-2}(x_t)\, \sigma_{\delta,t-1}^2(x_t)
\;\leq\;
C_2 \log\!\left(1 + \sigma_{\delta,0}^{-2}(x_t)\, \sigma_{\delta,t-1}^2(x_t)\right).
\]

Multiplying both sides by $\sigma_{\delta,0}^2(x_t)$ and summing over $t=1,\dots,T$ yields
\[
\sum_{t=1}^T \sigma_{\delta,t-1}^2(x_t)
\;\leq\; \sigma_{\delta,0}^2(x_t)\, C_2 \sum_{t=1}^T 
\log\!\left(1 + \sigma_{\delta,0}^{-2}(x_t)\, \sigma_{\delta,t-1}^2(x_t)\right).
\]
As in the proof of~\eqref{eq:gamma-bound}, the sum of logarithms is bounded by $2\gamma_{\delta,T}$, so
\[
\sum_{t=1}^T \sigma_{\delta,t-1}^2(x_t)
\;\leq\; \sigma_{\delta,0}^2(x_t)\, C_2 \cdot 2\gamma_{\delta,T}.
\]

Finally, since $\sigma_{\delta,0}^2(x_t) = \sigma_{g,N}^2(x_t) + \sigma^2$, we obtain
\[
\sum_{t=1}^T \sigma_{\delta,t-1}^2(x_t) 
\;\leq\; \Biggl(\frac{2 \gamma_{g,N}}{N - 2 \gamma_{g,N}}\,\sigma_0^2 + \sigma^2\Biggr) 
\cdot C_2 \cdot 2 \gamma_{\delta,T}.
\]

\textbf{Conclusion.}  
Combining the bounds for the $g$- and $\delta$- terms, we have
\[
\sum_{t=1}^T r_t^2 \;\leq\; 4 \beta_t \left( 
\frac{2T \gamma_{g,N}\sigma_0^2}{N - 2\gamma_{g,N}} 
+ 2 C_2 \gamma_{\delta,T}\Bigl(\tfrac{2 \gamma_{g,N}}{N - 2 \gamma_{g,N}}\,\sigma_0^2 + \sigma^2\Bigr)
\right).
\]

Finally, by the Cauchy--Schwarz inequality,
\[
R_T^2 \;\leq\; T \sum_{t=1}^T r_t^2,
\]
which completes the proof of Theorem~\ref{thm:regret-bound}.

\section{Additional Proofs}\label{app:add_proof}
\label{sec:additional-proofs}

\subsection{Proof of Corollary~\ref{corollary:asymptotic-bd}}

Since $\gamma_{g,N} = o(N)$, we have
\[
\frac{ \gamma_{g,N} }{ N - 2 \gamma_{g,N} } = \frac{\gamma_{g,N}}{N} \cdot \frac{1}{1-2(\gamma_{g,N}/N)} 
= \frac{\gamma_{g,N}}{N} \cdot \frac{1}{1 - o(1)} 
= \mathcal{O}\!\left( \frac{\gamma_{g,N}}{N} \right).
\]

From \eqref{eq:regret-bound-precise} and the assumption $\tau^2 = \mathcal{O}(\sigma^2)$, it follows that
\begin{equation}
\label{eq:proof-corollary:asymptotic-bd-1}
R_T \leq \mathcal{O}\!\left( \sqrt{ T \beta_T } \left( \frac{ \sigma^2_0 T \gamma_{g,N} }{ N } + \gamma_{\delta,T} \Big( \tfrac{\sigma^2_0 \gamma_{g,N}}{N} + \sigma^2 \Big) \right)^{1/2} \right).
\end{equation}

If
\[
\frac{ \gamma_{g,N} }{ N } = \mathcal{O}\!\left( \frac{ \gamma_{\delta,T} }{ T } \right),
\]
then, since $\gamma_{\delta,T} = \mathcal{O}(T)$, we also have 
\[
\frac{ \sigma^2_0 T \gamma_{g,N} }{ N } = \mathcal{O} \left( \sigma^2_0 \gamma_{\delta,T} \right), \quad
\sigma^2_0 \frac{ \gamma_{g,N} }{ N } = \mathcal{O}(\sigma^2_0).
\]
Substituting into \eqref{eq:proof-corollary:asymptotic-bd-1}, we obtain
\[
R_T = \mathcal{O}\! \left( \left(  \sigma^2 + \sigma^2_0 \right)^{\frac{1}{2}} \sqrt{ T \beta_T \gamma_{\delta,T} } \right).
\]

\subsection{Proof of Proposition~\ref{prop:gamma-rates-tau}}
\label{sec:proof-gamma-rates}

We now provide a proof of Proposition~\ref{prop:gamma-rates-tau}, which establishes the growth rates of the maximum information gain $\gamma_{\delta,T}$ for several common kernel classes.

\begin{proof}
Let $A \subset \mathcal{D}$ with $|A|=T$, and let $K_A$ be the kernel matrix associated with $k_\delta$. By construction,
\[
k_\delta(x,x') \;=\; \tau^2 \bar{k}_\delta(x,x'),
\]
so the eigenvalues of $K_A$ satisfy
\[
\lambda_i \;=\; \tau^2 \bar{\lambda}_i, \qquad i=1,2,\dots,T,
\]
where $\{\bar{\lambda}_i\}$ are the eigenvalues corresponding to $\bar{k}_\delta$.

The mutual information is given by
\[
I(y_A;f_A) = \frac{1}{2}\sum_{i=1}^T \log\!\Big(1 + \tfrac{\lambda_i}{\sigma^2}\Big)
= \frac{1}{2}\sum_{i=1}^T \log\!\Big(1 + \tfrac{\tau^2}{\sigma^2}\bar{\lambda}_i\Big).
\]
Maximizing over all $A \subset \mathcal{D}$ with $|A|=T$ gives the definition of $\gamma_{\delta,T}$.

\paragraph{Step 1. Reduction to eigenvalue tail bounds.}  
Following the approach of \citet{srinivas2010gaussian}, Theorem~4, we split the eigenvalues into the top $T^\star$ and the tail:
\[
\gamma_{\delta,T} \;\le\; \tfrac{1}{2}\!\sum_{i=1}^{T^\star} \log\!\Big(1+\tfrac{\tau^2}{\sigma^2}\bar{\lambda}_i\Big) 
\;+\; \tfrac{T}{2\sigma^2}\!\sum_{i>T^\star} \tau^2 \bar{\lambda}_i.
\]
This decomposition follows from bounding $\log(1+x) \le x$ for small eigenvalues in the tail.

\paragraph{Step 2. Eigen-decay of specific kernels.}
\begin{itemize}[leftmargin=2em]
\item \textbf{Linear kernel.}  
The spectrum has rank at most $d$, with eigenvalues bounded by $O(1)$. Thus,
\[
\gamma_{\delta,T} \;\le\; O\!\big(\tau^2 d \log(eT)\big) + O(\log(1+\tau^2)).
\]

\item \textbf{Squared Exponential (SE) kernel.}  
The eigenvalues of the normalized SE kernel decay exponentially in $i^{1/d}$~\citep{srinivas2010gaussian}. Optimizing $T^\star = O((\log T)^d)$ yields
\[
\gamma_{\delta,T} \;\le\; O\!\big(\tau^2 (\log T)^{d+1}\big) + O(\log(1+\tau^2)).
\]

\item \textbf{Matérn kernel.}  
The eigenvalues of the normalized Matérn kernel with smoothness $\nu$ decay polynomially as $\bar{\lambda}_i = O(i^{-\frac{2\nu+d}{d}})$~\citep{srinivas2010gaussian}. Optimizing $T^\star$ in the bound gives
\[
\gamma_{\delta,T} \;\le\; O\!\Big(\tau^2 T^{\frac{d(d+1)}{2\nu+d(d+1)}} \log T\Big) + O(\log(1+\tau^2)).
\]
\end{itemize}

\paragraph{Step 3. Collecting terms.}  
The additional $\log(1+\tau^2)$ term appears from bounding the contribution of the first few eigenvalues, which is independent of $T$. Combining the above establishes the claimed bounds for all three kernel families.
\end{proof}

\section{Additional Experimental Details}\label{sec:exp-detail}
This section provides the detailed configurations, including implementation settings and hyperparameter, and additional results in average regrets.

\subsection{Real-World Experimental Settings}
Here we describe the real-world experimental configurations, including algorithm implementation, dataset preparation, hyperparameter choices.

\subsubsection{Algorithm Implementation}
We used a single NVIDIA A40 GPU (48GB) for our experiments. In the Gradient boosting task, the observation noise for the source function $g$ is set to $\sigma_0 = 0.02$, and for the target function $f$ to $\sigma = 0.01$. 
Both functions employ a Matérn kernel with smoothness parameter $\nu = \tfrac{5}{2}$, which is kept consistent across all experiments.; 
the lengthscales are $1.8$ for the source and $1.0$ for the target. 
The difference function uses a squared exponential kernel with lengthscale $1.2$. The variance of the difference kernel is bounded by $\tau^2 = 0.2^2$. We fix $\beta_t = 0.2$ for all algorithms (GP-UCB, Env-GP, Diff-GP, and DeltaBO) to balance exploration and exploitation uniformly across all algorithms. In the Multi-layer perceptron task, the kernel choices and variance bound for the difference kernel remain the same, 
while the lengthscales are set to $2.0$, $1.0$, and $1.0$ for the source, target, and difference functions. 
For MLP task, we fix $\beta_t = 0.3$.

\subsubsection{Hyperparameter Settings}
Hyperparameters settings mainly follow \citet{liu2023global}. Hyperparameters may take either continuous or categorical forms. To ensure a consistent comparison between DeltaBO and Bayesian optimization baselines, 
we restrict all hyperparameter tuning experiments to a continuous search domain 
$[0,10]^{d}$. For categorical hyperparameters, we assign disjoint subintervals of equal length 
within this range to represent each category. 
For instance, consider the hyperparameter indicating whether to shuffle samples 
in each iteration (bool, True or False); 
we map the intervals $[0,5)$ and $[5,10]$ to the two options, respectively. 
Continuous hyperparameters are linearly scaled to the same $[0,10]$ range. 
As an example, if a hyperparameter originally takes values in $(0,1)$, 
we multiply its value by $10$ to obtain its mapped representation in $(0,10)$.

Hyperparameters for real-world task are listed as follows.

\textbf{Classification with Multi-Layer Perceptron.}
\begin{enumerate}
    \item Activation function (string, ``identity'', ``logistic'', ``tanh'', or ``relu'').
    \item Strength of the L2 regularization term (float, $[10^{-6}, 10^{-2}]$).
    \item Initial learning rate (float, $[10^{-6}, 10^{-2}]$).
    \item Maximum number of iterations (integer, $[100, 300]$).
    \item Whether to shuffle samples in each iteration (bool, True or False).
    \item Exponential decay rate for the first moment vector (float, $(0, 1)$).
    \item Exponential decay rate for the second moment vector (float, $(0, 1)$).
    \item Maximum number of epochs without tolerance improvement (integer, $[1, 10]$).
\end{enumerate}

\textbf{Classification with Gradient Boosting.}
\begin{enumerate}
    \item Loss function (string, ``logloss'' or ``exponential'').
    \item Learning rate (float, $(0, 1)$).
    \item Number of estimators (integer, $[20, 200]$).
    \item Fraction of samples used for fitting base learners (float, $(0, 1)$).
    \item Criterion to measure split quality (string, ``friedman\_mse'' or ``squared\_error'').
    \item Minimum number of samples required to split an internal node (integer, $[2, 10]$).
    \item Minimum number of samples required to be at a leaf node (integer, $[1, 10]$).
    \item Minimum weighted fraction of the total sum of weights (float, $(0, 0.5)$).
    \item Maximum depth of regression estimators (integer, $[1, 10]$).
    \item Number of features considered for best split (float, ``sqrt'' or ``log2'').
    \item Maximum number of leaf nodes in best-first fashion (integer, $[2, 10]$).
\end{enumerate}

\subsection{Synthetic Experimental Settings}
For Gaussian kernel functions, the lengthscale of all kernels is 0.1, with observation noise for both source ($\sigma_0^2$) and target ($\sigma^2$) being 0.01, and $\tau^2 = 0.3^2$. For Bohachevsky functions, lengthscale of source, target, and difference kernel is 1.6, 0.8, and 1.0 
with $\sigma_0^2= 0.24$, $\sigma^2 = 0.06$ and $\tau^2 = 0.3^2$. In assumption-satisfied setting, 
$\sigma_0^2= 0.1$ and $\sigma^2 = 0.01$. The target GP in the baseline algorithms is modeled using a Matérn kernel with lengthscale 1.0.  

\subsection{Additional Experimental Results in Average Regrets}
Figure~\ref{average_regret} shows the performances of all compared algorithms in average regrets. The experimental settings are exactly the same as those for cumulative regrets in the main paper, and similar performances can be observed.

\begin{figure*}[!htbp]  
\centering
\begin{subfigure}{0.48\textwidth}
    \includegraphics[width=\linewidth]{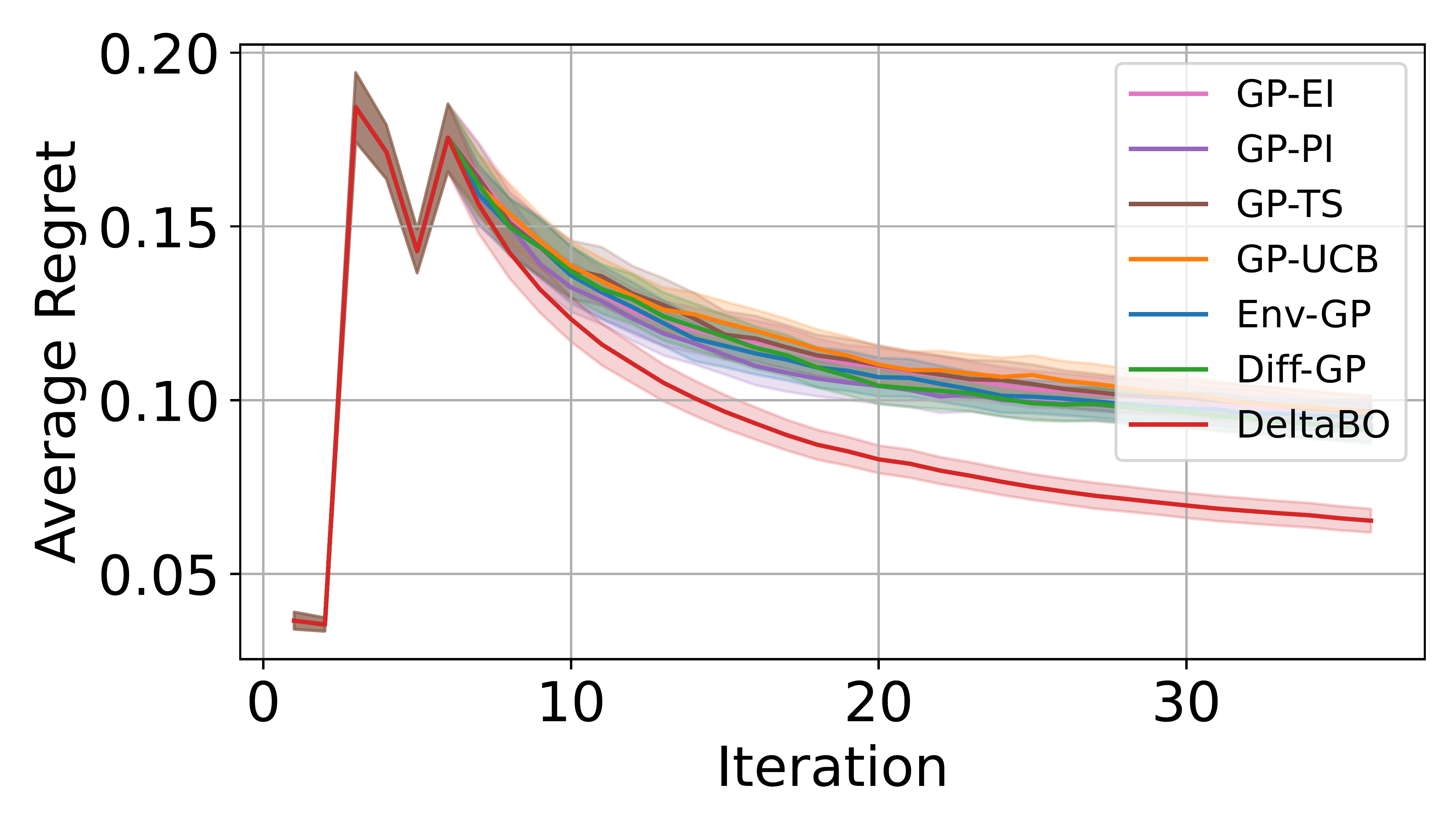}
    \caption{AutoML on GBoost}
    \label{fig:rw-XGB_a}
\end{subfigure}
\begin{subfigure}{0.48\textwidth}
    \includegraphics[width=\linewidth]{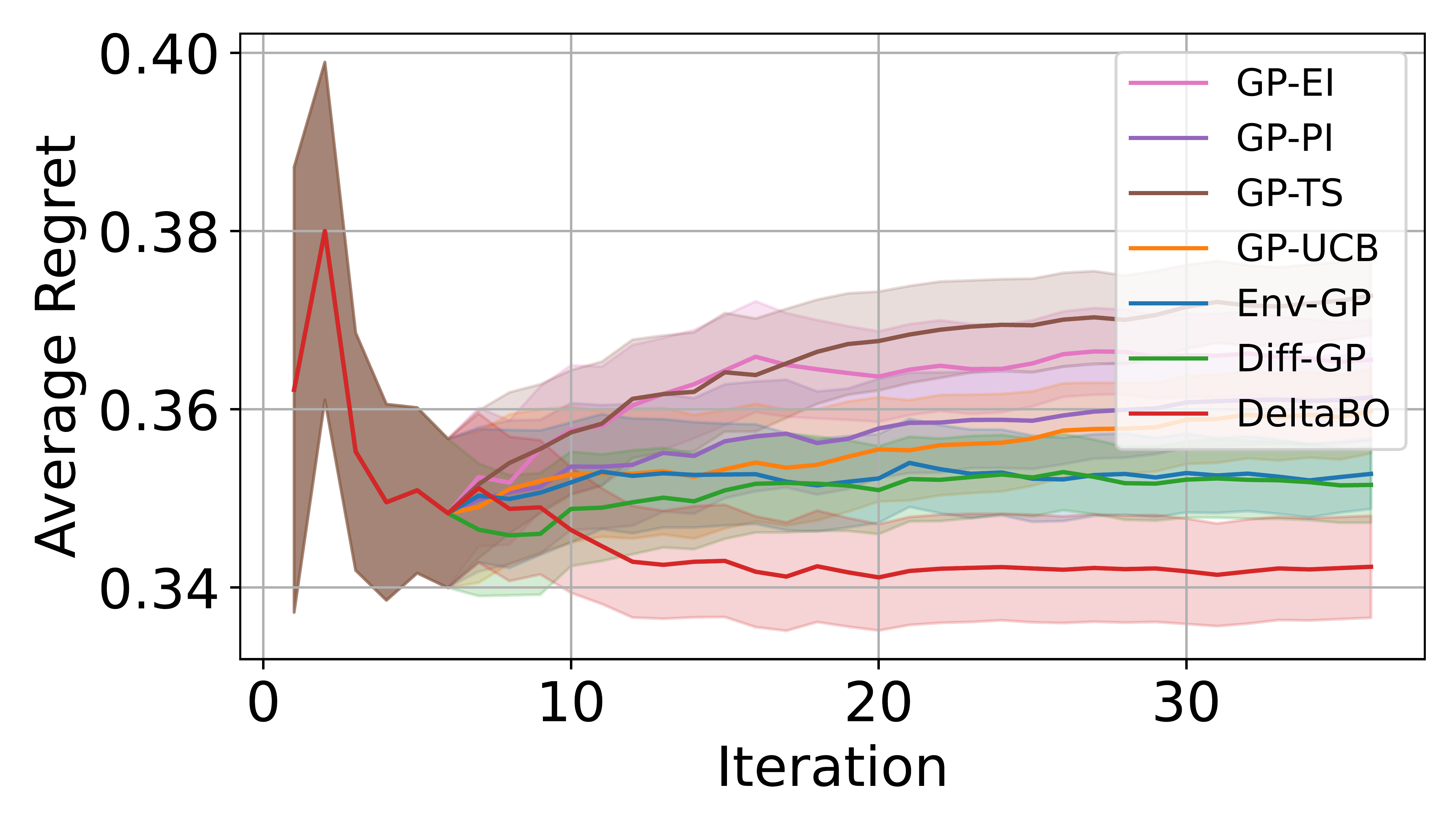}
    \caption{AutoML on MLP}
    \label{fig:rw-mlp_a}
\end{subfigure}
\begin{subfigure}{0.48\textwidth}
    \includegraphics[width=\linewidth]{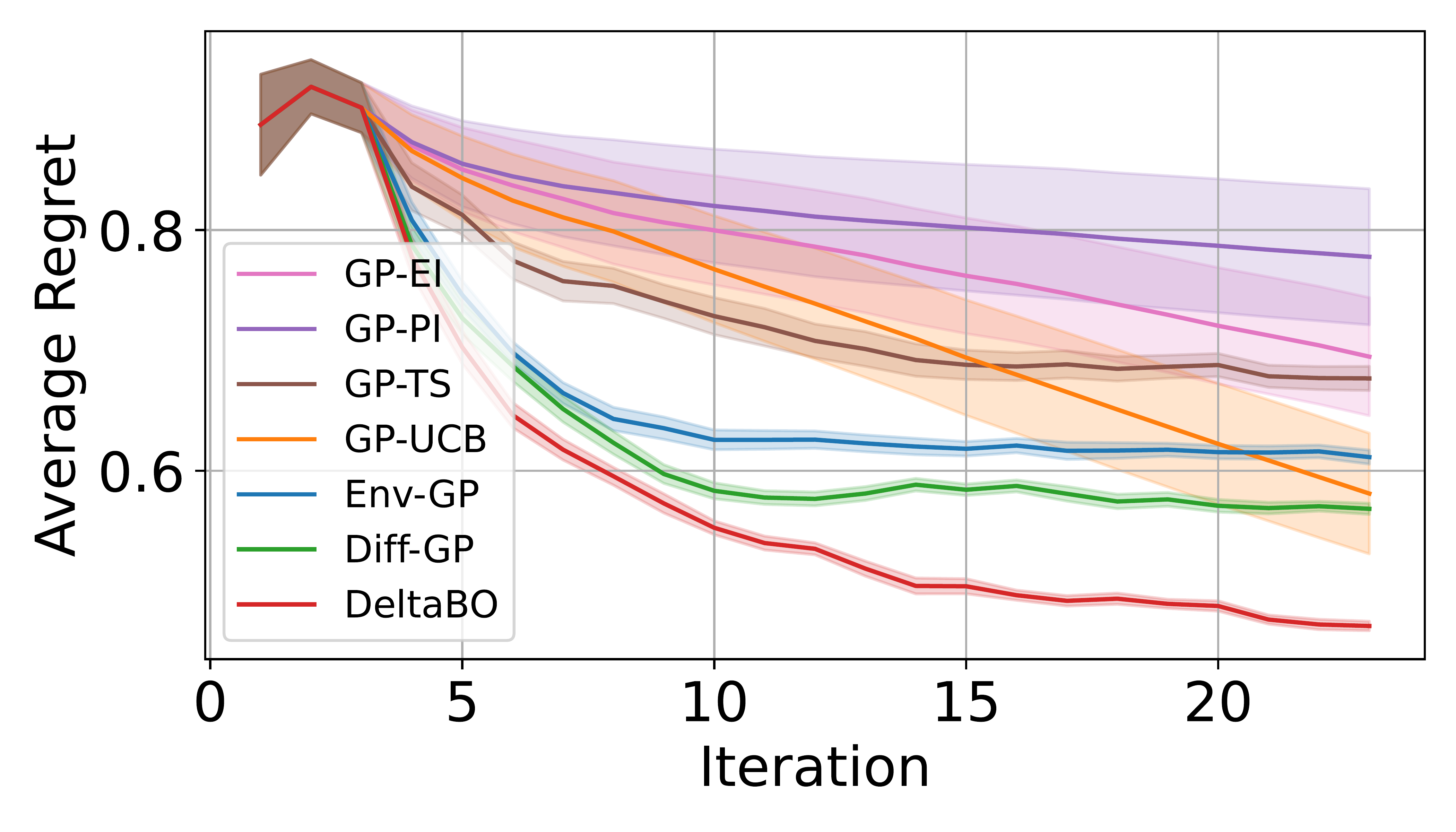}
    \caption{Synthetic Gaussain kernel}
    \label{fig:rw-gau_a}
\end{subfigure}
\begin{subfigure}{0.48\textwidth}
    \includegraphics[width=\linewidth]{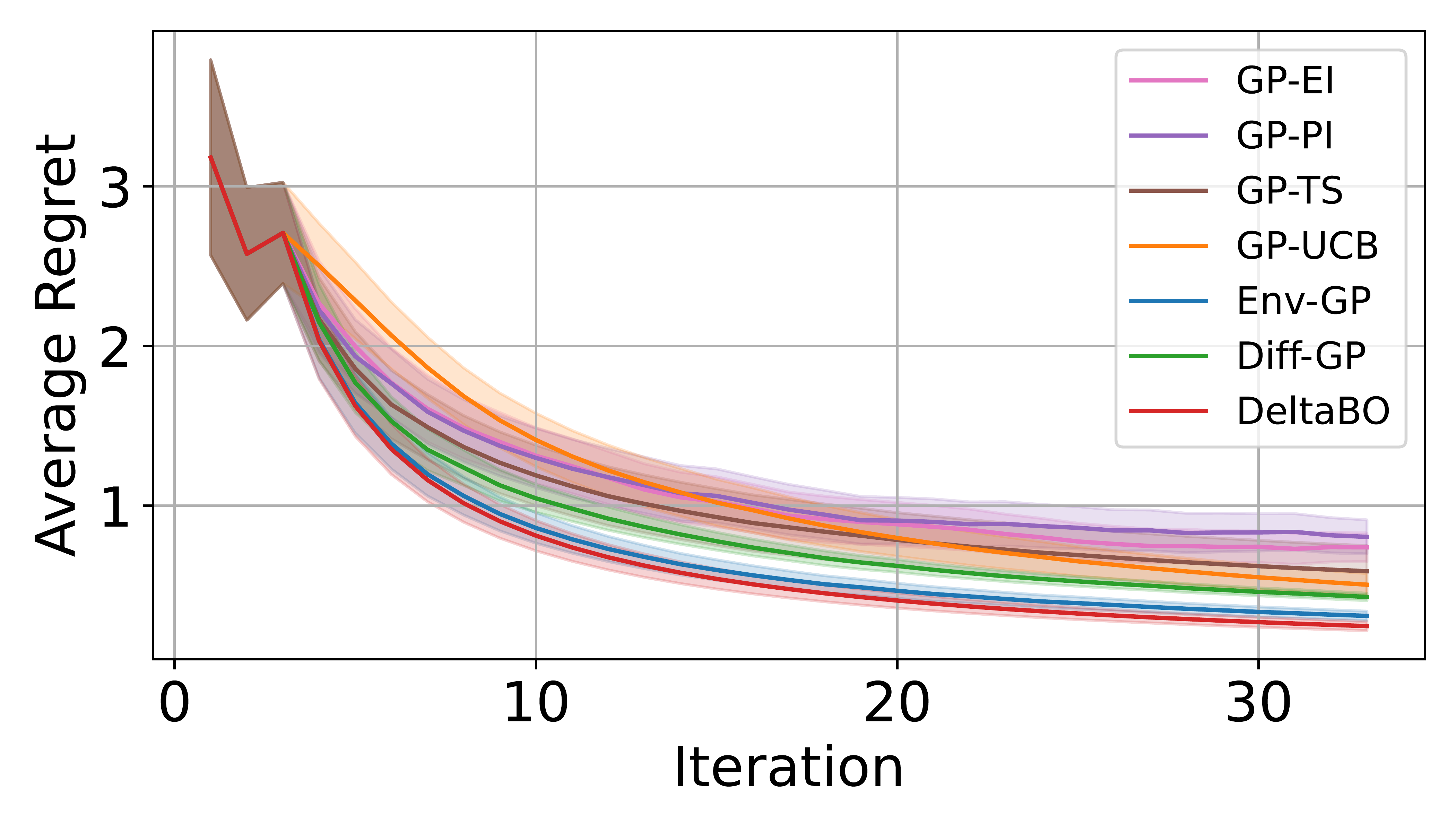}
    \caption{Synthetic Bohachevsky}
    \label{fig:rw-boh_a}
\end{subfigure}
\begin{subfigure}{0.48\textwidth}
    \includegraphics[width=\linewidth]{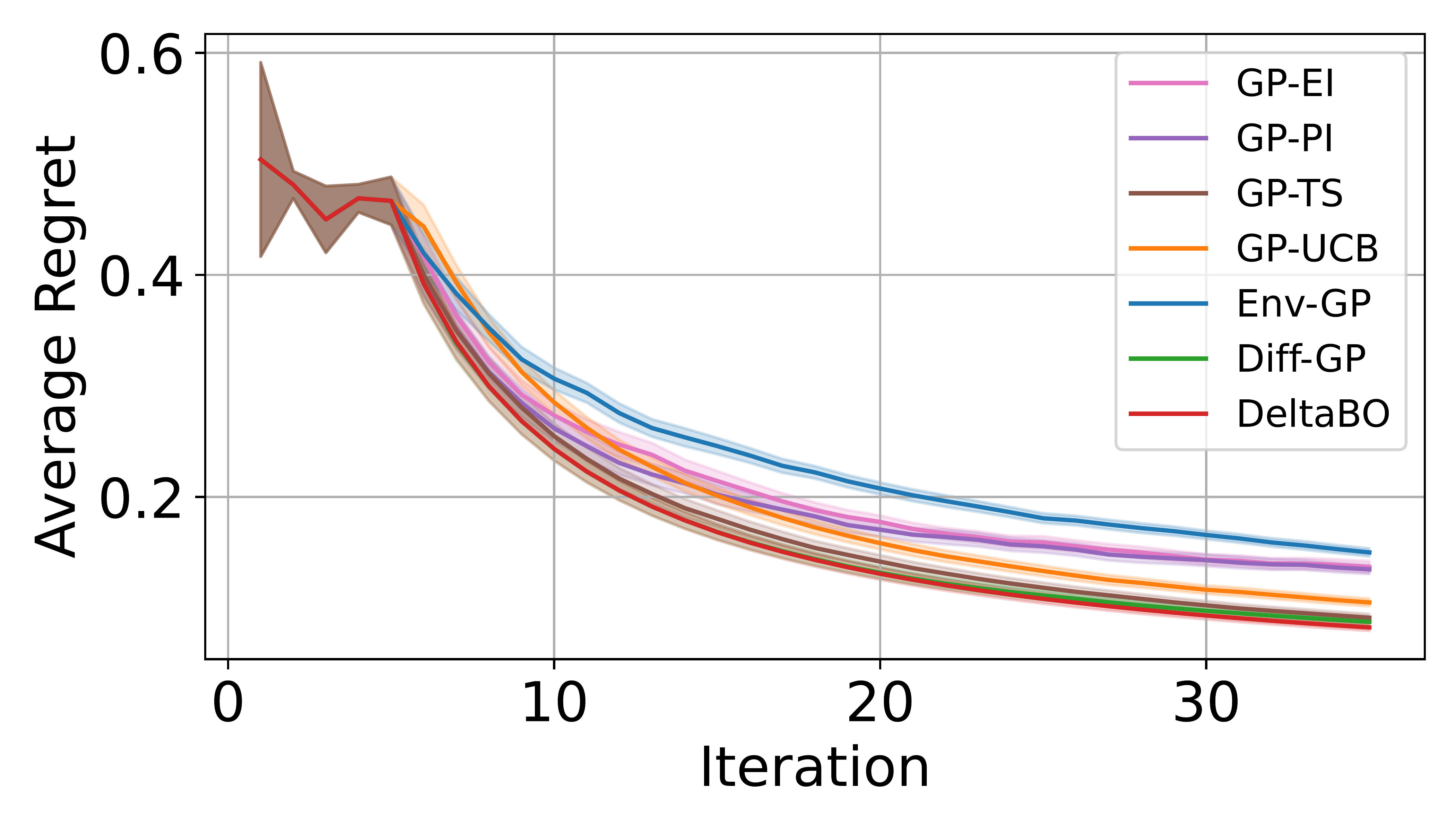}
    \caption{Assumptions-satisfied setting}
    \label{fig:rw-assumption_a}
\end{subfigure}
\caption{Average regrets of all compared algorithms.}
\label{average_regret}
\end{figure*}

\end{document}